\pdfoutput=1
\documentclass[pmlr]{jmlr}

\usepackage{longtable}

\usepackage{booktabs}
\usepackage[load-configurations=version-1]{siunitx} 

\theorembodyfont{\upshape}
\theoremheaderfont{\scshape}
\theorempostheader{:}
\theoremsep{\newline}

\jmlrvolume{}
\jmlryear{2017}
\jmlrworkshop{Algorithmic Learning Theory 2017}

\title[A Modular Analysis of Adaptive (Non-)Convex Optimization]{A Modular Analysis of Adaptive (Non-)Convex Optimization: Optimism, Composite Objectives, and Variational Bounds}

\author{
  \Name{Pooria Joulani} \Email{pooria@ualberta.ca}\\
    \addr Department of Computing Science\\
    University of Alberta\\
    Edmonton, Alberta, Canada
  \AND
  \Name{Andr\'{a}s Gy\"{o}rgy} \Email{a.gyorgy@imperial.ac.uk}\\
    \addr Department of Electrical and Electronic Engineering\\
    Imperial College London\\
    London, UK
  \AND
  \Name{Csaba Szepesv\'{a}ri} \Email{szepesva@ualberta.ca}\\
    \addr Department of Computing Science\\
    University of Alberta\\
    Edmonton, Alberta, Canada
}

\editors{Steve Hanneke and Lev Reyzin}

\usepackage{thmtools}

\declaretheorem[name=Lemma,refname={Lemma,Lemmas},Refname={Lemma,Lemmas},sibling=theorem]{lemma}

\declaretheorem[name=Corollary,refname={Corollary,Corollaries},Refname={Corollary,Corollaries},sibling=theorem]{corollary}

\declaretheorem[name=Assumption,refname={Assumption,Assumptions},Refname={Assumption,Assumptions}]{assumption}

\declaretheorem[name=Proposition,refname={Proposition,Propositions},Refname={Proposition,Propositions},sibling=theorem]{proposition}

\declaretheorem[name=Definition,refname={Definition,Definitions},Refname={Definition,Definitions},sibling=theorem]{definition}

\declaretheorem[name=Example,refname={Example,Examples},Refname={Example,Examples}]{example}

\declaretheorem[name=Remark,refname={Remark,Remarks},Refname={Remark,Remarks}]{remark}

\usepackage[utf8]{inputenc} 
\usepackage[T1]{fontenc}    
\usepackage{hyperref}       
\usepackage{url}            
\usepackage{booktabs}       
\usepackage{amsfonts}       
\usepackage{nicefrac}       
\usepackage{microtype}      

\usepackage{amsmath}
\usepackage{upgreek}

\usepackage{algorithm2e}
\usepackage{nccmath}

\usepackage{xspace}

\usepackage{enumerate}

\usepackage[capitalize]{cleveref}

\Crefname{corollary}{Corollary}{Corollaries}
\Crefname{lemma}{Lemma}{Lemmas}
\Crefname{proposition}{Proposition}{Propositions}
\Crefname{theorem}{Theorem}{Theorems}
\Crefname{definition}{Definition}{Definitions}
\Crefname{assumption}{Assumption}{Assumptions}
\Crefname{example}{Example}{Examples}
\Crefname{remark}{Remark}{Remarks}
\Crefname{setting}{Setting}{Settings}

\allowdisplaybreaks

\usepackage{float}
\floatstyle{boxed}
\restylefloat{figure}

\usepackage[disable]{todonotes}

 \newcommand{\argmin}{\operatorname{argmin}}

 \newcommand{\E}[1]{\mathbb{E}{\left\lbrace #1 \right\rbrace}}
 \newcommand{\EE}{\mathbb{E}}

 \newcommand{\Reals}{\mathbb{R}}
 \newcommand{\ExReals}{\overline{\mathbb{R}}}
 \newcommand{\R}{\mathbb{R}}

 \newcommand{\dotx}[1]{\langle #1 \rangle}

 \newcommand{\tl}[1]{\tilde{#1}}

 \newcommand{\dom}{\tn{dom}}

 \newcommand{\AdaFTRL}{\textsc{Ada-FTRL}\xspace}
 
 \newcommand{\AdaMD}{\textsc{Ada-MD}\xspace}
 \newcommand{\FTProx}{\textsc{FTRL-Prox}\xspace}
 \newcommand{\FTCentered}{\textsc{FTRL-Centered}\xspace}

 \newcommand{\AdaGrad}{\textsc{AdaGrad}\xspace}

 \newcommand{\tn}[1]{\textnormal{#1}}

 \newcommand{\cX}{\mathcal{X}}
 \newcommand{\cI}{\mathcal{I}}
 
 \newcommand{\cH}{\mathcal{H}}
 \newcommand{\cC}{\mathcal{C}}

 \newcommand{\cB}{\mathcal{B}}

 \newcommand{\trace}[1]{\textnormal{tr}\left(#1\right)}
 \newcommand{\tT}{\top}
 
 \newcommand{\norm}[1]{\|#1\|}
 \newcommand{\one}[1]{\mathbb{I}\{#1\}}

\DontPrintSemicolon
\RestyleAlgo{boxruled}
\LinesNumbered

\begin{document}

\maketitle

\begin{abstract}
Recently, much work has been done on extending the scope of online learning and incremental stochastic optimization algorithms. In this paper we contribute to this effort in two ways: First, based on a new regret decomposition and a generalization of Bregman divergences, we provide a self-contained, modular analysis of the two workhorses of online learning: (general) adaptive versions of Mirror Descent (MD) and the Follow-the-Regularized-Leader (FTRL) algorithms. The analysis is done with extra care so as not to introduce assumptions not needed in the proofs and allows to combine, in a straightforward way, different algorithmic ideas (e.g., adaptivity, optimism, implicit updates) and learning settings (e.g., strongly convex or composite objectives). This way we are able to reprove, extend and refine a large body of the literature, while keeping the proofs concise.
The second contribution is a byproduct of this careful analysis: We present algorithms with improved variational bounds for smooth, composite objectives, including a new family of optimistic MD algorithms with only one projection step per round. Furthermore, we provide a simple extension of adaptive regret bounds to practically relevant non-convex problem settings with essentially no extra effort.
\end{abstract}


\section{Introduction}
\label{sec:intro}
Online and stochastic optimization algorithms form the underlying machinery in much of modern machine learning.
Perhaps the most well-known example is Stochastic Gradient Descent (SGD) and its adaptive variants, the so-called \AdaGrad algorithms \citep{mcmahan2010adaptive,duchi2011adaptive}. Other special cases include multi-armed and linear bandit algorithms, as well as algorithms for online control, tracking and prediction with expert advice  \citep{cesa-bianchi2006prediction,shalev2011online,hazan2016introduction}.

There are numerous algorithmic variants in online and stochastic optimization, such as adaptive~\citep{duchi2011adaptive,mcmahan2010adaptive} and optimistic algorithms~\citep{rakhlin2013online,rakhlin2013optimization,chiang2012online,mohri2016accelerating,kamalaruban2016improved}, implicit updates~\citep{kivinen1997exponentiated,kulis2010implicit}, composite objectives~\citep{xiao2009dual,duchi2011adaptive,duchi2010composite}, or non-monotone regularization~\citep{sra2016adadelay}.
Each of these variants has been analyzed under a specific set of assumptions on the problem, e.g., smooth~\citep{juditsky2011solving,lan2012optimal,dekel2012optimal}, convex~\citep{shalev2011online,hazan2016introduction,orabona2015generalized,mcmahan2014survey}, or strongly convex~\citep{shalev2009mind,hazan2007logarithmic,orabona2015generalized,mcmahan2014survey} objectives.
However, a useful property is typically missing from the analyses: modularity.
It is typically not clear from the original analysis whether the algorithmic idea can be mixed with other techniques, or whether the effect of the assumptions extend beyond the specific setting considered.
For example, based on the existing analyses it is very much unclear to what extent \AdaGrad techniques, or the effects of smoothness, or variational bounds in online learning, extend to new learning settings. Thus, for every new combination of algorithmic ideas, or under every new learning setting, the algorithms are typically analyzed from scratch.

A special new learning setting is non-convex optimization. While the bulk of results in online and stochastic optimization assume the convexity of the loss functions, online and stochastic optimization algorithms have been successfully applied in settings where the objectives are non-convex.
In particular, the highly popular deep learning techniques \citep{goodfellow2016deep} are based on the application of stochastic optimization algorithms to non-convex objectives.
In the face of this discrepancy between the state of the art in theory and practice, an on-going thread of research attempts to generalize the analyses of stochastic optimization to non-convex settings.
In particular, certain non-convex problems have been shown to actually admit efficient optimization methods,
usually taking some form of a gradient method (one such problem is matrix completion, see, e.g., \citealp{ge2016nips,bhojanapalli2016nips}).

The goal of this paper is to provide a flexible, modular analysis of online and stochastic optimization algorithms that allows to easily combine different algorithmic techniques and learning settings under as little assumptions as possible.

\subsection{Contributions}\label{sec:contributions}
First, building on previous attempts to unify the analyses of online and stochastic optimization \citep{shalev2011online,hazan2016introduction,orabona2015generalized,mcmahan2014survey}, we provide a unified analysis of a large family of optimization algorithms in general Hilbert spaces. The analysis is crafted to be \emph{modular}: it decouples the contribution of each assumption or algorithmic idea from the analysis, so as to enable us to combine different assumptions and techniques without analyzing the algorithms from scratch.

The analysis depends on a novel decomposition of the optimization performance (optimization error or regret) into two parts: the first part captures the generic performance of the algorithm, whereas the second part connects the assumptions about the learning setting to the information given to the algorithm.
Lemma~\ref{lem:regret-decomposition} in Section~\ref{sec:regret-decomposition} provides such a decomposition.\footnote{This can be viewed as a refined version of the so-called ``be-the-leader'' style of analysis. Previous work (e.g., \citealt{mcmahan2014survey,shalev2011online}) may give the impression that ``follow-the-leader/be-the-leader'' analyses lose constant factors while other methods such as primal-dual analysis don't. This is not the case about our analysis. In fact, we improve constants in optimistic online learning; see Section~\ref{sec:optimistic-learning}.}
Then, in Theorem~\ref{thm:cheating-regret}, we bound the generic (first) part, using a careful analysis of the linear regret of generalized adaptive Follow-The-Regularized-Leader (FTRL) and Mirror Descent (MD) algorithms.

Second, we use this analysis framework to provide a concise summary of a large body of previous results. \cref{sec:applications} provides the basic results, and \cref{sec:composite-objective,sec:implicit-update,sec:optimistic-learning} present the relevant extensions and applications.

Third, building on the aforementioned modularity, we analyze new learning algorithms. In particular, in \cref{sec:scale-free} we analyze a new adaptive, optimistic, composite-objective FTRL algorithm with variational bounds for smooth convex loss functions, which combines the best properties and avoids the limitations of the previous work. We also present a new class of optimistic MD algorithms with only one MD update per round (\cref{sec:ao-md}).

Finally, we extend the previous results to special classes of non-convex optimization problems. In particular, for such problems, we provide global convergence guarantees for general adaptive online and stochastic optimization algorithms. The class of non-convex problems we consider (cf. \cref{sec:non-convex}) generalizes practical classes of functions considered in previous work on non-convex optimization.

\subsection{Notation and definitions}\label{sec:notation-defs}
We will work with a
(possibly infinite-dimensional) Hilbert space $\cH$ over the reals.
That is, $\cH$ is a real vector space equipped with an inner product $\dotx{\cdot , \cdot}: \cH \times \cH \to \Reals$, such that $\cH$ is complete with respect to (w.r.t.) the norm induced by $\dotx{\cdot,\cdot}$.
Examples include $\cH = \Reals^{d}$ (for a positive integer $d$) where $\dotx{\cdot, \cdot}$ is the standard dot-product, or $\cH = \Reals^{m \times n}$, the set of $m \times n$ real matrices, where $\dotx{A,B}=\trace{A^{\tT}B}$, or $\cH = \ell^{2}(\cC)$, the set of square-integrable real-valued functions on $\cC \subset \Reals^d$, where $\dotx{f,g} = \int_{\cC} f(x) g(x) dx$ for any $f,g \in \cH$.

We denote the extended real line by $\ExReals := \Reals \cup \{-\infty, +\infty\}$, and work with functions of the form $f:\cH \to \ExReals$.
Given a set $C \subset \cH$, the \emph{indicatrix} of $C$ is the function $\cI_{C}:\cH \to \ExReals$ given by $\cI_{C}(x) = 0$ for $x \in C$ and $\cI_{C}(x) = +\infty$ for $x \not \in C$. The \emph{effective domain} of a function $f: \cH \to \ExReals$, denoted by $\dom(f)$, is the set $\{x \in \cH\ | \ f(x) < +\infty\}$ where $f$ is less than infinity; conversely, we identify any function $f: C \to \ExReals$ defined only on a set $C \subset \cH$ by the function $f + \cI_{C}$. A function $f$ is \emph{proper} if $\dom(f)$ is non-empty and $f(x) > -\infty$ for all $x \in \cH$.

Let $f:\cH \to \ExReals$ be proper.
We denote the set of all sub-gradients of $f$ at $x \in \cH$ by $\partial f(x)$, i.e.,
\begin{align*}
\partial f(x) := \left\{ \ u \in \cH \ | \ \forall y \in \cH, \dotx{u, y-x} + f(x)  \le f(y) \ \right\}\,.
\end{align*}
The function $f$ is \emph{sub-differentiable} at $x$ if $\partial f(x) \not= \emptyset$; we use $f'(x)$ to denote any member of $\partial f(x)$.
Note $\partial f(x)=\emptyset$ when $x \not\in \dom(f)$.

Let $x \in \dom(f)$, assume that $f(x) > - \infty$, and let $z \in \cH$. The \emph{directional derivative} of $f$ at $x$ in the direction $z$ is defined as
$f'(x;z) := \lim_{\alpha \downarrow 0} \frac{f(x+\alpha z) - f(x)}{\alpha}\,,$ provided that the limit exists in $[-\infty, +\infty]$.
The function $f$ is \emph{differentiable} at $x$
if it has a \emph{gradient} at $x$, i.e., a vector $\nabla f(x) \in \cH$ such that $f'(x;z) = \dotx{\nabla f(x),z}$ for all $z \in \cH$.
The function $f$ is \emph{locally sub-differentiable} at $x$ if it has a \emph{local sub-gradient} at $x$, i.e., a vector $g_x \in \cH$ such that $\dotx{g_x,z} \le f'(x; z)$ for all $z \in \cH$. We denote the set of local sub-gradients of $f$ at $x$ by $\delta f(x)$. Note that if $f'(x; z)$ exists for all $z\in \cH$, and $f$ is sub-differentiable at $x$, then it is also locally sub-differentiable with $g_x = u$ for any $u \in \partial f(x)$. Similarly, if $f$ is differentiable at $x$, then it is also locally sub-differentiable, with $g_x = \nabla f(x)$.
The function $f$ is called \emph{directionally differentiable at $x \in \dom(f)$} if $f(x) > -\infty$ and $f'(x; z)$ exists in $[-\infty, +\infty]$ for all $z \in \cH$; $f$ is called \emph{directionally differentiable} if it is directionally differentiable at every $x \in \dom(f)$.

Next, we define a generalized\footnote{If $f$ is differentiable at $x$, then \eqref{eq:bregman-div} matches the traditional definition of Bregman divergence. Previous work also considered generalized Bregman divergences, e.g., the works of \citet{telgarsky2012agglomerative,kiwiel1997proximal} and the references therein. However, our definition is not limited to convex functions, allowing us to study convex and non-convex functions under a unified theory; see, e.g., \cref{sec:non-convex}.}
notion of Bregman divergence:
\begin{definition}[Bregman divergence]\label{def:bregman-div}
Let $f$ be directionally differentiable at $x \in \dom(f)$. The $f$-induced \emph{Bregman divergence} from $x$ is the function from $\cH \to \ExReals$, given by
\begin{align}
\cB_{f}(y,x) :=
\begin{cases}
  f(y) - f(x) - f'(x; y-x)\,. \label{eq:bregman-div} & \qquad \text{if $f(y)$ is finite;} \\
  +\infty & \qquad \text{otherwise}\,.
\end{cases}
\end{align}
\end{definition}
A function $f: \cH \to \ExReals$ is \emph{convex} if for all $x,y \in \dom(f)$ and all $\alpha \in (0,1)$, $\alpha f(x) + (1-\alpha) f(y) \ge f\left( \alpha x + (1-\alpha) y \right)$.
We can show that a proper convex functions is always directionally differentiable, and the Bregman divergence it induces is always nonnegative  (see Appendix~\ref{apx:tech-results}).
Let $\|.\|$ denote a norm on $\cH$ and let $L, \beta > 0$. A directionally differentiable function $f: \cH \to \ExReals$ is \emph{$\beta$-strongly convex} w.r.t. $\|.\|$ iff
$ \cB_{f}(x,y) \ge \frac{\beta}{2} \| x - y \|^{2}$
for all $x, y \in \dom(f)$. The function $f$ is \emph{$L$-smooth} w.r.t. $\|.\|$ iff for all $x, y \in \dom(f)$,
       $\left| \cB_{f}(x,y) \right| \le \frac{L}{2} \| x - y \|^{2}$.

We use $\{c_t\}_{t=i}^{j}$ to denote the sequence $c_i, c_{i+1}, \dots, c_{j}$, and $c_{i:j}$ to denote the sum $\sum_{t=i}^{j} c_t$, with $c_{i:j} := 0$ for $i>j$.

\section{Problem setting: online optimization}\label{sec:online-optimization}

We study a general first-order iterative optimization setting that encompasses several common optimization scenarios, including online, stochastic, and full-gradient optimization.
Consider a convex set $\cX \subset \cH$, a sequence of directionally differentiable
functions $f_1, f_2, \dots, f_T$ from $\cH$ to $\ExReals$ with $\cX \subset \dom(f_t)$ for all $t=1,2,\dots, T$, and a first-order iterative optimization algorithm. The algorithm starts with an initial point $x_1$. Then, in each iteration $t=1,2,\dots,T$, the algorithm suffers a loss $f_t(x_t)$ from the latest point $x_t$, receives some feedback $g_t \in \cH$, and selects the next point
$x_{t+1}$. Typically, $\dotx{g_t, \cdot}$ is supposed to be an estimate or lower bound on the directional derivative of $f_t$ at $x_t$.
This protocol is summarized in Figure~\ref{fig:online-optimization}.

Unlike Online Convex Optimization (OCO), at this stage we do not assume that the $f_t$ are convex\footnote{There is a long tradition of non-convex assumptions in the Stochastic Approximation (SA) literature, see, e.g., the book of \citet{bertsekas1978stochastic}. Our results differ in that they apply to more recent advances in online learning (e.g., AdaGrad algorithms), and we derive any-time regret bounds, rather than asymptotic convergence results, for specific non-convex function classes.} or differentiable, nor do we assume that $g_t$ are gradients or sub-gradients.
Our goal is to minimize the \emph{regret} $R_T(x^*)$ against any $x^* \in \cX$, defined as
$R_T(x^*) = \sum_{t=1}^{T} \left(f_t(x_t) - f_t(x^*)\right)\,.$

\begin{figure}[t]
  \textbf{Input}: convex set $\cX \subset \cH$; directionally differentiable functions $f_1, f_2, \dots, f_{T}$ from $\cH$ to $\ExReals$.
  \begin{itemize}
    \item The algorithm selects an initial point $x_1 \in \cX$.
    \item \textbf{For each} time step $t = 1,2,\dots,T$:
    \begin{itemize}
      \item The algorithm observes feedback $g_t \in \cH$ and
	  selects the next point $x_{t+1} \in \cX$.
    \end{itemize}
  \end{itemize}
  \textbf{Goal:} Minimize the regret $R_T(x^*)$ against any $x^* \in \cX$.
  \caption{Iterative optimization.
  \label{fig:online-optimization}}
\end{figure}

\subsection{Regret decomposition}\label{sec:regret-decomposition}

Below, we provide a decomposition of $R_T(x^*)$ (proved in Appendix~\ref{app:prL1}) which holds for any sequence of points $x_1, x_2, \dots, x_{T+1}$ and any $x^*$.
The decomposition is in terms of the \emph{forward linear regret} $R^{+}_T(x^*)$, defined as
\[
R^{+}_T(x^*) := \sum_{t=1}^{T} \dotx{g_t, x_{t+1} - x^*}\,.
\]
 Intuitively, $R^+_T$ is the regret (in linear losses) of the ``cheating'' algorithm that uses action $x_{t+1}$ at time $t$, and depends only on the choices of the algorithm and the feedback it receives.
\begin{lemma}[Regret decomposition]\label{lem:regret-decomposition} Let $x^*, x_1, x_2, \dots, x_{T+1}$ be any sequence of points in $\cX$. For $t=1,2,\dots, T$, let $f_t:\cH \to \ExReals$ be directionally differentiable with $\cX \subset \dom(f_t)$, and let $g_t \in \cH$. Then,
  \begin{align}
      R_T(x^*) = &\ R^{+}_T(x^*) + \sum_{t=1}^{T} \dotx{g_t, x_{t} - x_{t+1}} - \sum_{t=1}^{T} \cB_{f_t}(x^*, x_{t}) + \sum_{t=1}^{T} \delta_t\,, \label{eq:regret-decomposition}
  \end{align}
  where $\delta_t = -f'(x_t; x^* - x_t) + \dotx{g_t, x^* - x_t}$.
\end{lemma}
Intuitively, the second term captures the regret due to the algorithm's inability to look ahead into the future.\footnote{This is also related to the concept of ``prediction drift'', which appears in learning with delayed feedback \citep{joulani2016delay}, and to the role of stability in online algorithms \citep{saha2012interplay}.}
The last two terms capture, respectively, the gain in regret that is possible due to the curvature of $f_t\,$, and the accuracy of the first-order (gradient) information $g_t$.

In light of this lemma, controlling the regret reduces to controlling the individual terms in \eqref{eq:regret-decomposition}. First, we provide upper bounds on $R^{+}_T(x^*)$ for a large class of online algorithms.

\section{The algorithms: \AdaFTRL and \AdaMD}
\label{sec:algorithms}
In this section, we analyze \AdaFTRL and \AdaMD. These two algorithms generalize the well-known core algorithms of online optimization: FTRL \citep{shalev2011online,hazan2016introduction} and
MD \citep{nemirovsky1983problem,beck2003mirror,warmuth1997continuous,duchi2010composite}.
In particular, \AdaFTRL and \AdaMD capture variants of FTRL and MD such as
Dual-Averaging \citep{nesterov2009primal,xiao2009dual},
AdaGrad \citep{duchi2011adaptive,mcmahan2010adaptive},
composite-objective algorithms \citep{xiao2009dual,duchi2011adaptive,duchi2010composite},
implicit-update MD \citep{kivinen1997exponentiated,kulis2010implicit},
strongly-convex and non-linearized FTRL \citep{shalev2009mind,hazan2007logarithmic,orabona2015generalized,mcmahan2014survey},
optimistic FTRL and MD
\citep{rakhlin2013online,rakhlin2013optimization,chiang2012online,mohri2016accelerating,kamalaruban2016improved},
and even algorithms like AdaDelay \citep{sra2016adadelay} that violate the common non-decreasing regularization assumption existing in much of the previous work.

\subsection{\AdaFTRL: Generalized adaptive Follow-the-Regularized-Leader}
The \AdaFTRL algorithm works with two sequences of \emph{regularizers}, $p_1, p_2, \dots, p_T$ and $q_0, q_1, q_2, \dots, q_T$, where each $p_t$ and $q_t$ is a function from $\cH$ to $\ExReals$.
At time $t=0,1,2,\dots,T$, having received $(g_s)_{s=1}^{t}$, \AdaFTRL uses $g_{1:t}, p_{1:t}$ and $q_{0:t}$ to compute the next point $x_{t+1}$. The regularizers $p_t$ and $q_t$ can be built by \AdaFTRL in an online adaptive manner using the information generated up to the end of time step $t$ (including $g_t$ and $x_t$). In particular, we use $p_t$ to distinguish the ``proximal'' part of this adaptive regularization: for all $t=1,2,\dots, T$, we require that $p_t$ (but not necessarily $q_t$) be minimized over $\cX$ at $x_t$, that is\footnote{
Note that $x_t$ does not depend on $p_t$, but is rather computed using only $p_{1:t-1}$. Once $x_t$ is calculated, $p_t$ can be chosen so that \eqref{eq:ptcond} holds (and then used in computing $x_{t+1}$).
},
\begin{align}
  p_t(x_t) = \inf_{x \in \cX} p_t(x) < +\infty\,.
  \label{eq:ptcond}
\end{align}

With the definitions above, for $t=0,1,2,\dots, T$, \AdaFTRL selects $x_{t+1}$ such that
\begin{align}
 x_{t+1} \in \argmin_{x \in \cX}\ \dotx{g_{1:t}, x} + p_{1:t}(x) + q_{0:t}(x)\,.
 \label{eq:ftrl-update}
\end{align}
In particular, this means that the initial point $x_1$ satisfies\footnote{
The case of an arbitrary $x_1$ is equivalent to using, e.g., $q_0 \equiv 0$ (and changing $q_1$ correspondingly).}
\begin{align*}
  x_1 \in \argmin_{x \in \cX}\ q_0(x)\,.
\end{align*}
In addition, for notational convenience, we define $r_t := p_t + q_{t-1}, t=1,2,\dots,T$, so that
\begin{align}
   x_{t+1} \in \argmin_{x \in \cX}\ \dotx{g_{1:t}, x} + q_{t}(x) + r_{1:t}(x)\,.
   \label{eq:ftrl-update-with-r}
\end{align}
Finally, we need to make a minimal assumption to ensure that \AdaFTRL is well-defined.
\begin{assumption}[Well-posed \AdaFTRL]\label{assum:ftrl}
  The functions $q_0$ and $p_t, q_t, t=1,2,\dots,T,$ are proper. In addition, for all $t=0,1,\dots,T$, the $\argmin$ sets that define $x_{t+1}$ in \eqref{eq:ftrl-update} are non-empty, and their optimal values are finite. Finally, for all $t=1,2,\dots,T$, $r_{1:t}$ is directionally differentiable, and $p_t$ satisfies \eqref{eq:ptcond}.
\end{assumption}

Table~\ref{tbl:ftrl-special-cases}
provides examples of several special cases of \AdaFTRL. In particular, \AdaFTRL combines, unifies and considerably extends the two major types of FTRL algorithms previously considered in the literature, i.e., the so-called \FTCentered and \FTProx algorithms \citep{mcmahan2014survey} and their variants, as discussed in the subsequent sections.

\begin{table}[!ht]
\small
  \begin{tabular*}{\textwidth}{
    lll
    }
    \toprule
    Algorithm & Regularization & Notes, Conditions and Assumptions \\
    \midrule
    Online Gradient & $q_0 = \frac{1}{2\eta} \| . \|_{2}^{2}$ & $\cX = \Reals^{d}, \eta > 0$ \\
    Descent (OGD) & $q_t = p_t = 0, t \ge 1$ & Update: $x_{t+1} = x_t - \eta g_t$\\
    \midrule
    Dual Averaging & $q_t = \frac{\alpha_t}{2} \| . \|_2^2$ & $\alpha_{0:t} \ge 0, \alpha_t \ge 0\, \  (t \ge 1)$ \\
    (DA) & $p_t = 0$ & \\
    \midrule
    AdaGrad -  & $q_t = \frac{1}{2\eta} \|x\|_{(t)}^2 $ &
    $\|x\|_{(t)}^2 := x^{\top} (Q_{0:t}^{1/2} - Q_{0:t-1}^{1/2}) x$
    \\
    Dual Averaging & $Q_0 = \gamma I$ & $Q_{1:t} = \sum_{s=1}^{t} g_s g_s^{\top} $ (full-matrix update)
    \\
    & $p_t = 0$ & $Q_{1:t}^{(j,j)} = \sum_{s=1}^{t} g_{s,j}^2$ (diagonal-matrix update) \\
    \midrule
    \FTProx  & $q_t = 0$ & $Q_0 = 0$
    \\
    & $p_t = \frac{1}{2\eta} \| x-x_t \|_{(t)}^2$ & $Q_t$ and $\|\cdot\|_{(t)}$ as in AdaGrad-DA\\
    \midrule
    Composite- & $q_0 = \tl{q}_0$ & For adding composite-objective learning to \\
    Objective & $q_t = \psi_{t} + \tl{q}_{t}$ & any instance of \AdaFTRL (see also \cref{sec:composite-objective})\\
    Online Learning & $p_t = \tl{p}_t$ & $x_{t+1} = \argmin_{\cX} \dotx{g_{1:t}, x} + \psi_{1:t}(x) + \tl{p}_{1:t} (x) + \tl{q}_{0:t}(x) $\\
    \bottomrule
  \end{tabular*}
  \caption{Some special instances of \AdaFTRL; see also the survey of \citet{mcmahan2014survey}.}\label{tbl:ftrl-special-cases}
\end{table}

\subsection{\AdaMD: Generalized adaptive Mirror-Descent}
As in \AdaFTRL, the \AdaMD algorithm uses two sequences of regularizer functions from $\cH$ to $\ExReals$: $r_1, r_2, \dots, r_T$ and $q_0, q_1, \dots, q_T$. Further, we assume that the domains of $(r_t)$ are non-increasing, that is, $\dom(r_t) \subset \dom(r_{t-1})$ for $t=2,3,\ldots,T$. Again, $q_t, r_t$ can be created using the information generated by the end of time step $t$.
The initial point $x_1$ of \AdaMD satisfies\footnote{
The case of an arbitrary $x_1$ is equivalent to using, e.g., $q_0 \equiv 0$ (and changing $r_1$ correspondingly).}
\begin{align*}
  x_1 \in \argmin_{x \in \cX}\ q_0(x)\,.
\end{align*}
Furthermore, at time $t=1,2,\dots,T$, having observed $(g_s)_{s=1}^{t}$, \AdaMD uses $g_t, q_t$ and $r_{1:t}$ to select the point $x_{t+1}$ such that
\begin{align}
	x_{t+1} \in \argmin_{x \in \cX}\ \dotx{g_{t}, x} + q_t(x) + \cB_{r_{1:t}}(x, x_t)\,.
	\label{eq:md-update}
\end{align}
In addition, similarly to \AdaFTRL, we define $p_t:= r_t - q_{t-1}, t=1,2,\ldots,T$, though we do not require $p_t$ to be minimized at $x_t$ in \AdaMD\footnote{
We use the convention $(+\infty) - (+\infty) = +\infty$ in defining $p_t$.
}.

Finally, we present our assumption on the regularizers of \AdaMD. Compared to \AdaFTRL, we require a stronger assumption to ensure that \AdaMD is well-defined, and that the Bregman divergences in \eqref{eq:md-update} have a controlled behavior.
\begin{assumption}[Well-posed \AdaMD]\label{assum:md}
  The regularizers $q_t, r_t, t=1,2,\dots,T$, are proper, and $q_0$ is directionally differentiable. In addition, for all $t=0,1,\dots,T$, the $\argmin$ sets that define $x_{t+1}$ in \eqref{eq:md-update} are non-empty, and their optimal values are finite. Finally, for all $t=1,2,\dots,T$, $q_t, r_{1:t}$, and $r_{1:t} + q_t$ are directionally differentiable, $x_t \in \dom(r_{1:t})$, and $r'_{1:t}(x_t; \cdot )$ is linear in the directions inside $\dom(r_{1:t})$, i.e., there is a vector in $\cH$, denoted by $\nabla r_{1:t}(x_t)$, such that $r'_{1:t}(x_t, x-x_t) = \dotx{\nabla r_{1:t}(x_t), x-x_t}$ for all $x \in \dom(r_{1:t})$.
\end{assumption}
\begin{remark}
  Our results also hold under the weaker condition that $r'_{1:t}(x_t; \cdot - x_t)$ is \emph{concave}\footnote{
  Without such assumptions, a Bregman divergence term in $r'_{1:t}$ appears in the regret bound of \AdaMD. Concavity ensures that this term is not positive and can be dropped, greatly simplifying the bounds.
  }
  (rather than linear) on $\dom(r_{1:t})$. However, in case of a convex $r_{1:t}$, this weaker condition would again translate into having a linear $r'_{1:t}$, because a convex $r_{1:t}$ implies a convex $r'_{1:t}$ \citep[Proposition 17.2]{bauschke2011convex}. While we do not \emph{require} that $r_{1:t}$ be convex, all of our subsequent examples in the paper use convex $r_{1:t}$. Thus, in the interest of readability, we have made the stronger assumption of linear directional derivatives here.
\end{remark}
\begin{remark}
  Note that $r'_{1:t}$ needs to be linear only in the directions inside the domain of $r_{1:t}$. As such, we avoid the extra technical conditions required in previous work, e.g., that $r_{1:t}$ be a Legendre function to ensure $x_t$ remains in the interior of $\dom(r_{1:t})$ and $\nabla r_{1:t}(x_t)$ is well-defined.
\end{remark}

\subsection{Analysis of \AdaFTRL and \AdaMD}
Next we present a bound on the forward regret of \AdaFTRL and \AdaMD, and discuss its implications; the proof is provided in Appendix~\ref{apx:cheating-regret}.
\begin{theorem}[Forward regret of \AdaFTRL and \AdaMD]\label{thm:cheating-regret}
  For any $x^* \in \cX$ and any sequence of linear losses $\dotx{g_t, \cdot}, t=1,2,\dots,T$, the forward regret of \AdaFTRL under Assumption~\ref{assum:ftrl} satisfies
  \begin{align}
    R^{+}_T(x^*)
    \le
	  &\ \sum_{t=0}^{T} \left( q_t(x^*) - q_{t}(x_{t+1}) \right) + \sum_{t=1}^T \left( p_t(x^*) - p_t(x_t) \right) - \sum_{t=1}^{T} \cB_{r_{1:t}}(x_{t+1}, x_t) \,, 	\label{eq:ftrl-forward-regret}
  \end{align}
  whereas the forward regret of \AdaMD under Assumption~\ref{assum:md} satisfies
  \begin{align}
    R^{+}_T(x^*) \le
    &\ \sum_{t=0}^{T} \left( q_t(x^*) - q_t(x_{t+1}) \right) + \sum_{t=1}^{T} \cB_{p_t}(x^*, x_t) - \sum_{t=1}^{T} \cB_{r_{1:t}}(x_{t+1}, x_t)\,.	\label{eq:md-forward-regret}
  \end{align}
\end{theorem}

\begin{remark}
  \cref{thm:cheating-regret} does not require the regularizers to be  non-negative or (even non-strongly) convex.\footnote{Nevertheless, such assumptions are useful when combining the theorem with Lemma~\ref{lem:regret-decomposition}}
  Thus, \AdaFTRL and \AdaMD capture algorithmic ideas like a non-monotone regularization sequence as in AdaDelay \citep{sra2016adadelay},
  and \cref{thm:cheating-regret} allows us to extend these techniques to other settings; see also Section~\ref{sec:related-work}.
\end{remark}

\begin{remark}
  In practice, \AdaFTRL and \AdaMD need to pick a specific $x_{t+1}$ from the multiple possible optimal points in \eqref{eq:ftrl-update} and \eqref{eq:md-update}. The bounds of \cref{thm:cheating-regret} apply irrespective of the tie-breaking scheme.
\end{remark}

In subsequent sections, we show that the generality of \AdaFTRL and \AdaMD, together with the flexibility of \cref{assum:ftrl,assum:md}, considerably facilitates the handling of various algorithmic ideas and problem settings, and allows us to combine them without requiring a new analysis for each new combination.

\section{Recoveries and extensions}\label{sec:applications}

Lemma~\ref{lem:regret-decomposition} and Theorem~\ref{thm:cheating-regret} together immediately result in generic upper bounds on the regret, given in \eqref{eq:put-together-ftrl} and \eqref{eq:put-together-md} in Appendix~\ref{sec:standard_proofs}. Under different assumptions on the losses and regularizers,
these generic bounds directly translate into concrete bounds for specific learning settings. We explore these concrete bounds in the rest of this section.

First, we provide a list of the assumptions on the losses and the regularizers for different learning settings.\footnote{In fact, compared to previous work (e.g., the references listed in \cref{sec:intro} and \cref{sec:algorithms}), these are typically relaxed versions of the usual assumptions.}
We consider two special cases of the setting of Section~\ref{sec:online-optimization}: Online optimization and stochastic optimization.
In online optimization, we make the following assumption:
\begin{assumption}[Online optimization setting] \label{assum:online-opt}
  For $t=1,2,\dots,T$, $f_t$ is locally sub-differentiable, and $g_t$ is a local sub-gradient of $f_t$ at $x_t$.
\end{assumption}
Note that $f_t$ may be non-convex, and $g_t$ does not need to define a global lower-bound (i.e., be a sub-gradient) of $f_t$; see Section~\ref{sec:notation-defs} for the formal definition of local sub-gradients.

The stochastic optimization setting is concerned with minimizing a function $f$, defined by $f(x) := \EE_{\xi \sim D}{F(x, \xi)}$.
In this case the performance metric is redefined to be the expected stochastic regret, $\E{R_T(x^*)} = \E{\sum_{t=1}^{T} \left(f(x_t) - f(x^*)\right)}$.\footnote{Indeed, in stochastic optimization the goal is to find an estimate $\hat{x}_{T}$ such that $\E{f(\hat{x}_{T})}-f(x^*)$ is small. It is well-known (e.g., \citealt[Theorem 5.1]{shalev2011online}) that for any $f$, this equals $\E{R_T(x^*)/T}$ if $\hat{x}_T$ is selected uniformly from $x_1,\ldots,x_T$. Also, if $f$ is convex, $\E{f(\hat{x}_{T})}-f(x^*) \le \E{R_T(x^*)/T}$ if
$\hat{x}_T$ is the average of $x_1,\ldots,x_T$ (such averaging can also be used with $\tau$-star convex functions, cf. Section~\ref{sec:tau-star-convex}).  Thus, analyzing the regret is satisfactory.}
Typically, if $F$ is differentiable in $x$, then $g_t = \nabla F(x_t,\xi_t)$, where $\xi_t$ is a random variable, e.g., sampled independently from $D$. In parallel to \cref{assum:online-opt}, we summarize our assumptions for this setting is as follows:
\begin{assumption}[Stochastic optimization setting] \label{assum:stoch-opt}
  The function $f$ (defined above) is locally sub-differentiable,  $f_t=f$ for all $t=1,2,\dots,T$,
  and
  $g_t$ is, in expectation, a local sub-gradient of $f$ at $x_t$: $\E{g_t | x_t} \in \delta f(x_t)$.
\end{assumption}
Again, it is enough for $g_t$ to be an unbiased estimate of a local sub-gradient (Section~\ref{sec:notation-defs}).

In both settings we will rely on the non-negativity of the loss divergences at $x^*$:
\begin{assumption}[Nonnegative loss-divergence]\label{assum:positive-bregman-ft}
  For all $t=1,2,\dots,T$, $\cB_{f_t}(x^*, x_t) \ge 0$.
\end{assumption}
It is well known that this assumption is satisfied when each $f_t$ is convex. However, as we shall see in \cref{sec:non-convex}, this condition also holds for certain classes of non-convex functions (e.g., star-convex functions and more). In the stochastic optimization setting, since $f_t = f$, this condition boils down to $\cB_f(x^*,x_t)\ge 0$, $t=1,2,\dots,T$.

\begin{figure}[t]
  \centering
  \includegraphics[width=0.9\textwidth]{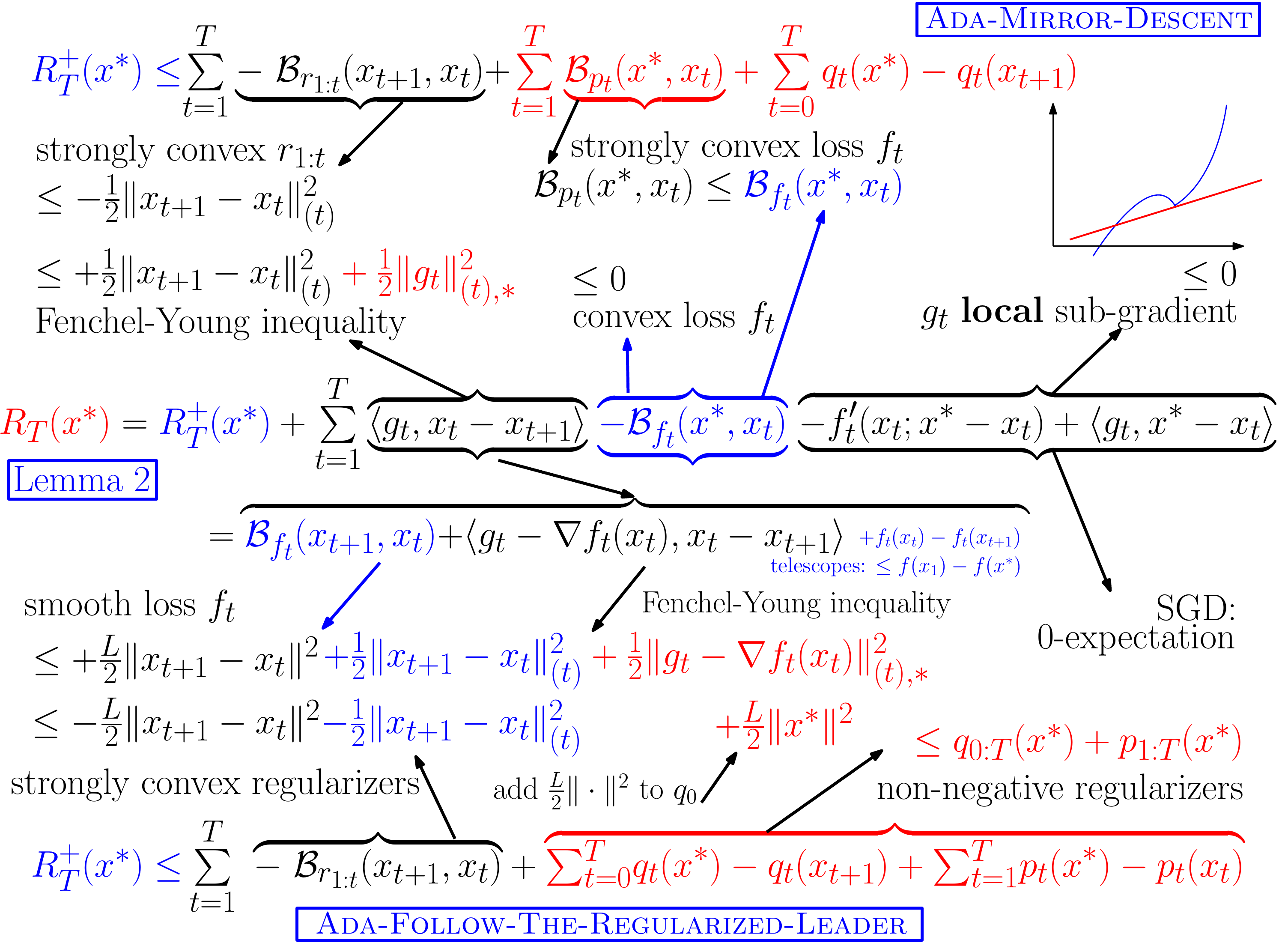}
  \caption{A summary of the proof techniques that incorporate each of the assumptions into regret bounds; see Corollaries~\ref{cor:recoveries}--\ref{cor:recoveries-stoch-smooth} and \cref{tbl:summary-of-results}.
  The identity in the middle is from \cref{lem:regret-decomposition}, whereas the top and bottom bounds on $R_T^+$ are due to \cref{thm:cheating-regret}. Each arrow shows the transformation of one of the terms, using the stated assumption or technique. The matching terms cancel, and the terms shown in red appear in the final bounds.}\label{fig:proofs}
\end{figure}

In both settings, the regret can be reduced when the losses are strongly convex. Furthermore, in the stochastic optimization setting, the smoothness of the loss is also helpful in decreasing the regret. The next two assumptions capture these conditions.
\begin{assumption}[Loss smoothness]\label{assum:smooth-f}
  The function $f$ is differentiable and $1$-smooth w.r.t. some norm $\|\cdot\|$.
\end{assumption}
\begin{assumption}[Loss strong convexity]\label{assum:strong-cvx-loss}
  The losses are 1-strongly convex w.r.t.\ the regularizers, that is, $\cB_{f_t}(x^*,x_t) \ge \cB_{p_t}(x^*,x_t)$ for $t=1,2,\dots,T$.
\end{assumption}
Note that if $p_t,q_{t-1}$ are convex, then it suffices to have $\cB_{r_t}$ in the condition (rather than $\cB_{p_t}$). Typically, if $f_t$ is strongly convex w.r.t. a norm $\|\cdot\|_{(t)}$, then $p_t$ (or $r_t$) is set to $\eta \|\cdot\|_{(t)}$ for some $\eta > 0$. Again, in stochastic optimization, \cref{assum:strong-cvx-loss} simplifies to $\cB_{f}(x^*,x_t)\ge \cB_{p_t}(x^*,x_t)$, $t=1,2,\dots,T$. Furthermore, if $p_t$ is convex, then \cref{assum:strong-cvx-loss} implies that $f_t$ is convex.

Finally, the results that we recover depend on the assumption that the total regularization, in both \AdaFTRL and \AdaMD, is strongly convex:
\begin{assumption}[Strong convexity of regularizers]\label{assum:strong-cvx-regularizer}
  For all $t=1,2,\dots,T$, $r_{1:t}$ is $1$-strongly convex w.r.t.\ some norm $\|\cdot\|_{(t)}$.
\end{assumption}

\begin{table}[t]
\centering
\scalebox{1}{\small
  \begin{tabular*}{\textwidth}{p{0.325\textwidth} p{0.13\textwidth} p{0.49\textwidth}}
    \toprule
    Setting / Algorithms & Assumptions & Regret / Expected Stochastic Regret Bound\\
    \midrule
    OO/SO\newline \AdaFTRL &
    \ref{assum:ftrl}, \ref{assum:online-opt}/\ref{assum:stoch-opt},
    \newline
    \ref{assum:positive-bregman-ft}, \ref{assum:strong-cvx-regularizer}
    &
    {$\sum_{t=0}^{T} \left( q_t(x^*) - q_{t}(x_{t+1}) \right)$}
    \newline
    {$ + \sum_{t=1}^{T} \left( p_t(x^*) - p_t(x_t) \right) + \sum_{t=1}^{T} \frac{1}{2} \|g_t\|_{(t),*}^2$}
    \\
    \midrule
    OO/SO\newline \AdaMD &
    \ref{assum:md}, \ref{assum:online-opt}/\ref{assum:stoch-opt},
    \newline
    \ref{assum:positive-bregman-ft}, \ref{assum:strong-cvx-regularizer}
    &
    {$\sum_{t=0}^{T} \left( q_t(x^*) - q_{t}(x_{t+1}) \right)$}
    \newline
    {$ + \sum_{t=1}^{T} \cB_{p_t}(x^*,x_t) + \frac{1}{2} \|g_t\|_{(t),*}^2$}
     \\
    \midrule
    Strongly-convex OO/SO\newline \AdaMD &
    \ref{assum:md}, \ref{assum:online-opt}/\ref{assum:stoch-opt},
    \newline
    (\ref{assum:positive-bregman-ft}), \ref{assum:strong-cvx-regularizer}, \ref{assum:strong-cvx-loss}
    &
    {$\sum_{t=0}^{T} \left( q_t(x^*) - q_{t}(x_{t+1}) \right)$}
    \newline
    {$ + \sum_{t=1}^{T} \frac{1}{2} \|g_t\|_{(t),*}^2$}
     \\
    \midrule
     Smooth SO\newline \AdaFTRL &
    \ref{assum:ftrl},  \ref{assum:stoch-opt}, \ref{assum:smooth-f},
    \newline
    \ref{assum:positive-bregman-ft}, \ref{assum:strong-cvx-regularizer}'
    &
    {$\frac{1}{2} \|x^*\|^2 + D + \sum_{t=0}^{T} \left( q_t(x^*) - q_{t}(x_{t+1}) \right)$}
    \newline
    {$ + \sum_{t=1}^{T} \left( p_t(x^*) - p_t(x_t) \right) + \sum_{t=1}^{T} \frac{1}{2} \|\sigma_t\|_{(t),*}^2$}
     \\
    \midrule
     Smooth SO\newline \AdaMD &
     \ref{assum:md}, \ref{assum:stoch-opt}, \ref{assum:smooth-f},
     \newline
     \ref{assum:positive-bregman-ft}, \ref{assum:strong-cvx-regularizer}'
     &
     {$\frac{1}{2} \|x^*-x_1\|^2 + D + \sum_{t=0}^{T} \left( q_t(x^*) - q_{t}(x_{t+1}) \right)$}
     \newline
     {$ + \sum_{t=1}^{T} \cB_{p_t}(x^*,x_t) + \frac{1}{2}  \|\sigma_t\|_{(t),*}^2$}
    \\
    \midrule
    Smooth \& strongly-convex SO\newline \AdaMD &
    \ref{assum:md}, \ref{assum:stoch-opt}, \ref{assum:smooth-f},
    \newline
    (\ref{assum:positive-bregman-ft}), \ref{assum:strong-cvx-regularizer}', \ref{assum:strong-cvx-loss}
    &
    {$\frac{1}{2} \|x^* - x_1\|^2 + D + \sum_{t=0}^{T} \left( q_t(x^*) - q_{t}(x_{t+1}) \right)$}
    \newline
    {$ + \sum_{t=1}^{T} \frac{1}{2} \|\sigma_t\|_{(t),*}^2$}\\
    \bottomrule
  \end{tabular*}}
  \caption{Recovered and generalized standard results for online optimization (OO) and stochastic optimization (SO); see Corollaries~\ref{cor:recoveries}--\ref{cor:recoveries-stoch-smooth}. A number in parentheses indicates that the assumption is not directly required, but is implied by the other assumptions.
  In the bounds above, $\sigma_t := g_t-\nabla f(x_t)$, $D:=f(x_1) - \inf_{\cX} f(x)$, and \ref{assum:strong-cvx-regularizer}' refers to a slightly modified version of Assumption~\ref{assum:strong-cvx-regularizer}, as described in \cref{cor:recoveries-stoch-smooth}.
  Note that setting $q_T=0$ recovers the \emph{off-by-one} property \citep{mcmahan2014survey} in \FTCentered vs. \FTProx; \AdaMD exhibits a similar property.
  For composite-objective \AdaFTRL and \AdaMD, these same bounds apply with $\tl{q}_t$ in place of $q_t$ in the summation; see \cref{tbl:ftrl-special-cases} and \cref{sec:composite-objective}.}
  \label{tbl:summary-of-results}
\end{table}

Table~\ref{tbl:summary-of-results} provides a summary of the standard results, under different sub-sets of the assumptions above, that are recovered and generalized using our framework. The derivations of these results are provided in the form of three corollaries in Appendix~\ref{sec:standard_proofs}. Note that the analysis is absolutely modular: each assumption is simply plugged into \eqref{eq:put-together-ftrl} or \eqref{eq:put-together-md} to obtain the final bounds, without the need for a separate analysis of \AdaFTRL and \AdaMD for each individual setting. A schematic view of the (standard) proof ideas is given in Figure~\ref{fig:proofs}.

\section{Composite-objective learning and optimization}
\label{sec:composite-objective}
Next, we consider the composite-objective online learning setting. In this setting, the functions $f_t$, from which the (local sub-)gradients $g_t$ are generated and fed to the algorithm, comprise only part of the loss. Instead of $R_T(x^*)$, we are interested in minimizing the regret
\begin{align*}
  R_T^{(\ell)}(x^*) := \sum_{t=1}^{T} f_t(x_t) + \psi_t(x_t) - f_t(x^*) - \psi_t(x^*) = R_T(x^*) + \sum_{t=1}^{T} \psi_t(x_t) - \psi_t(x^*) \,,
\end{align*}
using the feedback $g_t \in \delta f_t(x_t)$, where $\psi_t:\cH \to \ExReals$ are proper functions. The functions $\psi_t$ are not linearized, but are passed directly to the algorithm.

\newcounter{para}
Naturally, one can use the $q_t$ regularizers to pass the $\psi_t$ functions to \AdaFTRL and \AdaMD. Then, we can obtain the exact same bounds as in \cref{tbl:summary-of-results} on the composite regret $R^{(\ell)}_T(x^*)$; this recovers and extends the corresponding bounds by \citet{xiao2009dual,duchi2011adaptive,duchi2010composite,mcmahan2014survey}. In particular, consider the following two scenarios:
\refstepcounter{para}
\paragraph{Setting \thepara: $\psi_{t}$ is known before predicting $x_{t}$.}
\label{senarios:psi-t-known}
In this case, we run \AdaFTRL or \AdaMD with $q_t = \psi_{t+1} + \tl{q}_t, t=0,1,2,\dots,T$ (where $\psi_{T+1}:=0$). Thus, we have the update
\begin{align}
   x_{t+1} \in \argmin_{x \in \cX}\ \dotx{g_{1:t}, x} + \psi_{1:t+1}(x) + \tl{q}_{0:t}(x) + p_{1:t}(x)\,,
   \label{eq:composite-ftrl-update-known-psi}
\end{align}
for \AdaFTRL, and
\begin{align}
	x_{t+1} \in \argmin_{x \in \cX}\ \dotx{g_{t}, x} + \psi_{t+1}(x) + \tl{q}_t(x) + \cB_{r_{1:t}}(x, x_t)\,,
	\label{eq:composite-md-update-known-psi}
\end{align}
for \AdaMD. Then, we have the following result.
\begin{corollary}\label{cor:composite-obj-known-psi}
  Suppose that the iterates $x_1, x_2, \dots, x_{T+1}$ are given by the \AdaFTRL update \eqref{eq:composite-ftrl-update-known-psi} or the \AdaMD update \eqref{eq:composite-md-update-known-psi},
  and $q_t, p_t,$ and $r_t$  satisfy \cref{assum:ftrl} for \AdaFTRL, or \cref{assum:md} for \AdaMD.
  Then, under the conditions of each section of \cref{cor:recoveries,cor:recoveries-stoch,cor:recoveries-stoch-smooth}, the composite regret $R^{(\ell)}_T(x^*)$ enjoys the same bound as $R_T(x^*)$,
  but with $\sum_{t=0}^{T} \tl{q}_t(x^*) - \tl{q}_t(x_{t+1})$ in place of $\sum_{t=0}^{T} q_t(x^*) - q_t(x_{t+1})$.
\end{corollary}
\begin{proof}
  By definition, $\sum_{t=0}^{T} \tl{q}_t(x^*) - \tl{q}_t(x_{t+1}) -q_t(x^*) + q_t(x_{t+1}) = \sum_{t=1}^{T} \psi_t(x_t) - \psi_t(x^*)$.
  Thus, $R^{(\ell)}_T(x^*) = R_T(x^*) + \sum_{t=0}^{T} \tl{q}_t(x^*) - \tl{q}_t(x_{t+1}) -q_t(x^*) + q_t(x_{t+1})$.
  Upper-bounding $R_T(x^*)$ by the aforementioned corollaries completes the proof.
\end{proof}
\refstepcounter{para}
\paragraph{Setting \thepara: $\psi_t$ is revealed after predicting $x_t$, together with $g_t$.}
\label{senarios:psi-t-revealed}
In this case, we run \AdaFTRL and \AdaMD with functions $q_0 = \tl{q}_0$, $q_t = \psi_{t} + \tl{q}_t, t=1,2,\dots,T$, so that
\begin{align}
   x_{t+1} \in \argmin_{x \in \cX}\ \dotx{g_{1:t}, x} + \psi_{1:t}(x) + \tl{q}_{0:t}(x) + p_{1:t}(x)\,,
   \label{eq:composite-ftrl-update-hidden-psi}
\end{align}
for \AdaFTRL, and
\begin{align}
	x_{t+1} \in \argmin_{x \in \cX}\ \dotx{g_{t}, x} + \psi_{t}(x) + \tl{q}_t(x) + \cB_{r_{1:t}}(x, x_t)\,,
	\label{eq:composite-md-update-hidden-psi}
\end{align}
for \AdaMD. Then, we have the following result, proved in \cref{apx:composite-proofs}.
\begin{corollary}\label{cor:composite-obj-hidden-psi}
  Suppose that the iterates $x_1, x_2, \dots, x_{T+1}$ are given by the \AdaFTRL update \eqref{eq:composite-ftrl-update-hidden-psi} or the \AdaMD update \eqref{eq:composite-md-update-hidden-psi}, and $q_t, p_t,$ and $r_t$  satisfy \cref{assum:ftrl} for \AdaFTRL, or \cref{assum:md} for \AdaMD.
  Also, assume that $\psi_1(x_1) = 0$ and the $\psi_t$ are non-negative and non-increasing, i.e., that $\psi_1 \ge \psi_2 \ge \dots \ge \psi_{T+1} := 0$.\footnote{This relaxes the assumption in the literature, e.g., by \citet{mcmahan2014survey}, that $\psi_t = \alpha_t \psi$ for some fixed, non-negative $\psi$ minimized at $x_1$, and a non-increasing sequence $\alpha_t > 0$ (e.g., $\alpha_t = 1$); see also Setting~\ref{senarios:psi-t-known}.}
  Then, under the conditions of each section of \cref{cor:recoveries,cor:recoveries-stoch,cor:recoveries-stoch-smooth}, the composite regret $R^{(\ell)}_T(x^*)$ enjoys the same bound as $R_T(x^*)$, but with $\sum_{t=0}^{T} \tl{q}_t(x^*) - \tl{q}_t(x_{t+1})$ in place of $\sum_{t=0}^{T} q_t(x^*) - q_t(x_{t+1})$.
\end{corollary}

\begin{remark}
In both settings, the functions  $\psi_t$ are passed as part of the  regularizers $q_t$. Thus, if the $\psi_t$ are strongly convex, less additional regularization is needed in \AdaFTRL to ensure the strong convexity of $r_{1:t}$ because $q_{0:t-1}$ will already have some strongly convex components. In addition, in \AdaMD, when the $\psi_t$ are convex, the $\cB_{p_t}$ terms in \eqref{eq:md-forward-regret} will be smaller than the $\cB_{r_t}$ terms found in previous analyses of MD. This is especially useful for implicit updates, as shown in the next section. This also demonstrates another benefit of the generalized Bregman divergence: the $\psi_t$, and hence the $p_t$, may be non-smooth in general.
\end{remark}

\section{Implicit-update \AdaMD and non-linearized \AdaFTRL}
\label{sec:implicit-update}
Other learning settings can be captured using the idea of passing information to the algorithm using the $q_t$ functions. This information could include, for example, the curvature of the loss. In particular, consider the composite-objective \AdaFTRL and \AdaMD, and for $t=1,2,\dots,T$, let $\ell_t$ be a differentiable loss, $f_t = \dotx{\nabla \ell_t(x_t), x - x_t}$, and $\psi_t = \cB_{\ell_t}(\cdot, x_t)$.\footnote{
For non-differentiable $\ell_t$, let $f_t = \dotx{g_t, \cdot}$ and $\psi_t = \cB_{\ell_t}(\cdot, x_t) + \ell'(x_t; \cdot - x_t) - \dotx{g_t, \cdot}$ to get the same effect.
}
Then, $\ell_t = f_t + \psi_t$, $g_t = \nabla f_t(x_t) = \nabla \ell_t(x_t)$, and the composite-objective \AdaFTRL update \eqref{eq:composite-ftrl-update-hidden-psi} is equivalent to
\begin{align}
   x_{t+1} \in \argmin_{x \in \cX}\ \ell_{1:t}(x) + \tl{q}_{0:t}(x) + p_{1:t}(x)\,.
   \label{eq:implicit-ftrl-update}
\end{align}
Thus, non-linearized FTRL, studied by \citet{mcmahan2014survey}, is a special case of \AdaFTRL. With the same $f_t, \psi_t$, the composite-objective \AdaMD update~\eqref{eq:composite-md-update-hidden-psi} is equivalent to
\begin{align}
	x_{t+1} \in \argmin_{x \in \cX}\ \ell_{t}(x) + \tl{q}_t(x) + \cB_{r_{1:t}}(x, x_t)\,,
	\label{eq:implicit-md-update}
\end{align}
so the implicit-update MD is also a special case of \AdaMD.

Again, combining \cref{lem:regret-decomposition,thm:cheating-regret} results in a compact analysis of these algorithmic ideas.
In particular, for both updates \eqref{eq:implicit-ftrl-update} and \eqref{eq:implicit-md-update}, the bounds of \eqref{eq:put-together-ftrl} and \eqref{eq:put-together-md} apply on the regret in $f_t$.
Then, moving the terms $\psi_t(x^*)$ to the left turns each bound to a bound on the regret in $\ell_t$.
Furthermore, the terms $ - \psi_t(x_{t+1}) = - \cB_{\ell_t}(x_{t+1}, x_t)$ that remained on the right-hand side can be merged with the $-\cB_{r_{1:t}}(x_{t+1}, x_t)$ terms.
Thus, instead of $r_{0:t}$, it is enough for $r_{1:t} + \ell_t$ to be strongly convex w.r.t. the norm $\| \cdot \|_{(t)}$ (see the proofs of \cref{cor:recoveries,cor:recoveries-stoch,cor:recoveries-stoch-smooth}).
This means that if $\ell_t$ are strongly convex, then no further regularization is required: $\ell_{1:t}$ is strongly convex, and we get back the well-known logarithmic bounds for strongly-convex FTRL (Follow-The-Leader) and implicit-update MD
\citep{kivinen1997exponentiated,kulis2010implicit,shalev2009mind,hazan2007logarithmic,orabona2015generalized,mcmahan2014survey}.
In addition, as mentioned before, convexity of $\ell_t$ further reduces the term $\cB_{p_t}$ in implicit-update MD.

Finally, note that multiple pieces of information can be passed to the algorithm through $q_t$. In particular, none of the above interfere with further use of another composite term $\phi_t$ and obtaining regret bounds on $\ell_t + \phi_t$, as discussed in \cref{sec:composite-objective}.

\section{Adaptive optimistic learning and variational bounds}
\label{sec:optimistic-learning}
The goal of optimistic online learning algorithms \citep{rakhlin2013online,rakhlin2013optimization} is to obtain improved regret bounds when playing against ``easy'' (i.e., predictable) sequences of losses. This includes algorithms with regret rates that grow with the total ``variation'', i.e., the sum of the norms of the differences between pairs of consecutive losses $f_t$ and $f_{t+1}$ observed in the loss sequence: the regret will be small if the loss sequence changes slowly \citep{chiang2012online}.

Recently, \citet{mohri2016accelerating} proposed an interesting comprehensive framework for analyzing adaptive FTRL algorithms for predictable sequences. The framework has also been extended to MD by \citet{kamalaruban2016improved}.
However, despite their generality, the regret analyses of \citet{mohri2016accelerating} and \citet{kamalaruban2016improved} can be strengthened.
Specifically, the two analyses do not recover the variation-based results of \citet{chiang2012online} for smooth losses. In addition, their treatment of composite objectives introduces complications, e.g., only applies to Setting~\ref{senarios:psi-t-known} of \cref{sec:composite-objective} where $\psi_t$ is known before selecting $x_t$.

The flexibility of the framework introduced in this paper allows us to alleviate these and other limitations. In particular, we cast the Adaptive Optimistic FTRL (AO-FTRL) algorithm of \citet{mohri2016accelerating} as a special case of \AdaFTRL, and obtain a much simpler form of Adaptive Optimistic MD (AO-MD) as a special case of \AdaMD. Then, we strengthen and simplify the corresponding analyses, and recover and extend the results of \citet{chiang2012online}.
Finally, building on the modularity of our framework, we obtain an adaptive composite-objective algorithm with variational bounds that improves upon the results of \citet{mohri2016accelerating,kamalaruban2016improved,chiang2012online,rakhlin2013online,rakhlin2013optimization}.

\subsection{Adaptive optimistic FTRL}\label{sec:ao-ftrl}

Consider the online optimization setting of \cref{sec:applications} (\cref{assum:online-opt}). Suppose that the losses $f_1, f_2,\dots, f_T$ satisfy \cref{assum:positive-bregman-ft} (e.g., they are convex), and the sequence of points $x_{t+1}, t=0,1,2,\dots,T$ is given by
\begin{align*}
 x_{t+1} \in \argmin_{x \in \cX}\ \dotx{g_{1:t}+\tl{g}_{t+1}, x} + p_{1:t}(x) + \tl{q}_{0:t}(x)\,,
\end{align*}
where $\tl{g}_{t}, t=1,2,\dots, T+1,$ is any sequence of vectors in $\cH$. That is, we run \AdaFTRL, but we also incorporate $\tl{g}_{t+1}$ as a ``guess'' of the future loss $g_{t+1}$ that the algorithm will suffer. \citet{mohri2016accelerating} refer to this algorithm as AO-FTRL.

It is easy to see that AO-FTRL is a special case of \AdaFTRL: Define $\tl{g}_0 := 0$,\footnote{
This is different from the restriction in \citet{mohri2016accelerating} that $\tl{g}_1$ be $0$; we do not require that restriction. In particular, we allow $x_1$ to depend on $\tl{g}_1$, which can be arbitrary.
}
and for $t=0,1,\dots,T$, let $q_t = \tl{q}_t + \dotx{\tl{g}_{t+1}-\tl{g}_t, \cdot}$. Then, $q_{0:t} = \tl{q}_{0:t} + \dotx{\tl{g}_{t+1}, \cdot}$, so we have
\begin{align*}
  x_{t+1} \in \argmin_{x \in \cX}\ \dotx{g_{1:t}, x} + p_{1:t}(x) + q_{0:t}(x)\,,
\end{align*}
which is the \AdaFTRL update with this specific choice of $q_t$. Thus, the exact same manipulations as in \cref{cor:recoveries} give the following theorem, proved in \cref{apx:optimistic-learning-proofs}:
\begin{theorem}\label{thm:ao-ftrl-bound}
  If the losses satisfy \cref{assum:positive-bregman-ft}, and the regularizers $q_0$ and $p_t, q_t, t=1,2,\dots,T$, satisfy \cref{assum:ftrl,assum:strong-cvx-regularizer}, then the regret of AO-FTRL is bounded as
  \begin{align}
  R_T(x^*) \le &\ \sum_{t=0}^{T-1} \left( \tl{q}_t(x^*) - \tl{q}_{t}(x_{t+1}) \right) + \sum_{t=1}^T \left( p_t(x^*) - p_t(x_t) \right) + \sum_{t=1}^{T} \frac{1}{2} \|g_t - \tl{g}_{t} \|_{(t),*}^2\,.
  \label{eq:ao-ftrl-bound}
  \end{align}
\end{theorem}

This bound recovers Theorems 1 and 2 of \citet{mohri2016accelerating}. Similarly, one could prove parallels of \cref{cor:recoveries-stoch,cor:recoveries-stoch-smooth} for AO-FTRL. Then, the modularity property allows us (as we do in \cref{sec:scale-free}) to apply the composite-objective technique of \cref{sec:composite-objective} and recover Theorems 3-7 of \citet{mohri2016accelerating} (and hence their corollaries).  Indeed, the resulting analysis simplifies and improves on the analysis of \citet{mohri2016accelerating} in several aspects: we do not need to separate the cases for \FTProx and FTRL-General, we naturally handle the composite objective case for Settings \ref{senarios:psi-t-known} and \ref{senarios:psi-t-revealed} while avoiding any complications with proximal regularizers, and do not lose the constant $1/2$ factor. Finally, \cref{thm:cheating-regret} allows us to improve on the results of \citet{chiang2012online}, as we show next.

\subsection{Adaptive optimistic MD}
\label{sec:ao-md}
Interestingly, we can use the exact same assignment $q_t = \tl{q}_t + \dotx{\tl{g}_{t+1}-\tl{g}_t, \cdot}$ in \AdaMD. This results in the update
\begin{align*}
  x_{t+1} \in \argmin_{x \in \cX}\ \dotx{g_{t}+\tl{g}_{t+1}-\tl{g}_t, x} + \tl{q}_t(x) + \cB_{r_{1:t}}(x, x_t)\,.
\end{align*}
Applying the same argument as in \cref{thm:ao-ftrl-bound}, one can show that this optimistic MD algorithm enjoys the regret bound of \eqref{eq:ao-ftrl-bound} with the $p_t(x^*) - p_t(x_t)$ terms replaced by $\cB_{p_t}(x^*, x_t)$. This gives an optimistic MD algorithm with only one projection in each round; all other formulations \citep{kamalaruban2016improved,rakhlin2013optimization,rakhlin2013online,chiang2012online} require two MD steps in each round. This new formulation has the potential to greatly simplify the previous analyses of variants of optimistic MD. In particular, handling implicit updates or composite terms is a matter of including them in $\tl{q}_t$. Especially, unlike \citet{kamalaruban2016improved}, we can handle Setting~\ref{senarios:psi-t-revealed} in the exact same way as we do in the AO-FTRL case (see \cref{sec:scale-free}). Further exploration of the properties of this new class of algorithms is left for for future work.

\subsection{Variation-based bounds for online learning}
Suppose that the losses $f_t$ are differentiable and convex, and define $f_0:=0$. For any norm $\| \cdot \|$, we define the total variation of the losses in $\| \cdot \|_{*}^2$ as
\begin{align}
  D_{\| \cdot \|} := \sum_{t=1}^{T} \sup_{x \in \cX} \| \nabla f_t(x) - \nabla f_{t-1}(x) \|_{*}^2\,.
  \label{eq:d2-variation}
\end{align}
\citet{chiang2012online} use an optimistic MD algorithm to obtain regret bounds of order $O(\sqrt{D_2})$, where $D_2 = D_{\| \cdot \|_{2}}$, for linear as well as smooth losses.

If the losses are linear, i.e., $f_t = \dotx{g_t, \cdot}$, then \cref{thm:ao-ftrl-bound} immediately recovers the result of \citet[Theorem~8]{chiang2012online}.
In particular, let $\tl{q}_0 = (1/2\eta) \|.\|_2^2$, and for $t=1,2,\dots,T$, let $p_t=\tl{q}_t=0$, $\|.\|_{(t)}^2 = \eta \|.\|_2^2$, and $\tilde{g}_t =  g_{t-1}$.
Then \eqref{eq:ao-ftrl-bound} gives the regret bound $(\eta/2) \|x^*\|_2^2 + (1/(2\eta)) D_2$.
If $\|x\|_2 \le 1$ and we set $\eta$ based on $D_2$ (as Chiang et al. assume), we obtain their $O(\sqrt{D_2})$ bound.

If the losses are not linear but are $L$-smooth, then by the combination of \cref{lem:regret-decomposition} and \cref{thm:cheating-regret}, we still obtain $\sqrt{D_{\|\cdot\|}}$-bounds, as \citet[Theorem~10]{chiang2012online} also obtain for $D_2$.
This is because, unlike the analysis of \citet{mohri2016accelerating}, we retain the negative terms $-B_{r_{1:t}}(x_{t+1}, x_t)$ (essentially having the same role as the $B_t$ terms of \citealp{chiang2012online}) in the regret bound. Combined with ideas from Lemma 13 of \cite{chiang2012online}, this gives the desired bounds in terms of $D_{\|\cdot\|}$, proved in \cref{apx:optimistic-learning-proofs}:
\begin{theorem}
  Consider the conditions of \cref{thm:ao-ftrl-bound}, and further suppose that the losses $f_t$ are convex and $L$-smooth w.r.t. a norm $\| \cdot \|$.
  For $t=1,2,\dots,T+1$, let $\eta_t > 0$, and suppose that \cref{assum:strong-cvx-regularizer} holds with $\| \cdot \|_{(t)}^2 = \eta_t \| \cdot \|^2$.
  Further assume that $q_0 \ge 0$, $p_t,q_t \ge 0, t \ge 1$, and $\eta_t \eta_{t+1} \ge 8 L^2, t=1,2,\dots,T$.
  Then, AO-FTRL with $\tl{g}_t = g_{t-1}$ satisfies
  \begin{align}
    R_T(x^*) \le &\ \tl{q}_{0:T}(x^*) + p_{1:T}(x^*) + 2 \sum_{t=1}^{T} \frac{1}{\eta_t} \max_{x \in \cX} \|\nabla f_t(x) - \nabla f_{t-1}(x)\|_{*}^2\,.
    \label{eq:chiang-smooth-bound}
  \end{align}
  \label{thm:chiang-smooth-bound}
\end{theorem}
Letting $\eta_t = \eta = \sqrt{D_{\| \cdot \|}}$, and $\tl{q}_0 = \eta \| \cdot \|^2, \tl{q}_t, p_t = 0, t \ge 1,$ generalizes the $O(\sqrt{D_2})$ bound of \citet{chiang2012online} to any norm (under the same assumption they make, that $D_{\| \cdot \|} \ge 8L^2$). In the next section, we provide an algorithm that does not need prior knowledge of $D_{\|\cdot\|}$.

\subsection{Adaptive optimistic composite-objective learning with variational bounds}
\label{sec:scale-free}
Next, we provide a simple analysis of the composite-objective version of AO-FTRL, and obtain variational bounds in terms of $D_{\|\cdot\|}$ for composite objectives with smooth $f_t$. We focus on Setting~\ref{senarios:psi-t-revealed}; similar results are immediate for Setting~\ref{senarios:psi-t-known}. Consider the update
\begin{align}
  x_{t+1} \in \argmin_{x \in \cX}\ \dotx{g_{1:t}+\tl{g}_{t+1}, x} + \psi_{1:t}(x) + p_{1:t}(x) + \tl{q}_{0:t}(x)\,,
  \label{eq:composite-ao-ftrl-update}
\end{align}
that is, the composite-objective AO-FTRL algorithm. Then we have the following corollary of \cref{thm:ao-ftrl-bound}.

\begin{corollary}\label{cor:composite-ao-ftrl}
  Suppose that $\psi_t, t=1,2,\dots,T$, satisfy the conditions of \cref{cor:composite-obj-hidden-psi}, and $\tl{q}_0$ and $p_t, \tl{q}_t, t \ge 1,$ are non-negative.
  Let $q_0 = \tl{q}_0 + \dotx{\tl{g}_1, \cdot}$ and $q_t = \tl{q}_t + \psi_t + \dotx{\tl{g}_{t+1} - \tl{g}_t, \cdot}$.
  Suppose that $q_0$, $p_t, q_t, t \ge 1$ satisfy \cref{assum:ftrl,assum:strong-cvx-regularizer}, and $f_t, t \ge 1$ satisfy \cref{assum:positive-bregman-ft}.
  Then, composite-objective AO-FTRL (update \eqref{eq:composite-ao-ftrl-update}) satisfies
  \begin{align*}
    R^{(\ell)}_T(x^*) \le &\ \tl{q}_{0:T-1}(x^*) + p_{1:T}(x^*) + \sum_{t=1}^{T} \frac{1}{2} \|g_t - \tl{g}_{t} \|_{(t),*}^2\,.
  \end{align*}
\end{corollary}
\begin{proof}
  Starting as in \cref{cor:composite-obj-hidden-psi}, defining $\tl{g}_0 = 0$, and noting that $0 = q_0 - \tl{q}_0 - \dotx{\tl{g}_1, \cdot}$,
  \begin{align*}
    R^{(\ell)}_T(x^*)
    \le
    &\
    R_T(x^*) + \sum_{t=0}^{T} q_t(x_{t+1}) - \tl{q}_t(x_{t+1}) + \tl{q}_t(x^*) - q_t(x^*) - \dotx{\tl{g}_{t+1} - \tl{g}_t , x_{t+1} - x^*}\,.
  \end{align*}
  Proceeding as in \cref{thm:ao-ftrl-bound} completes the proof.
\end{proof}

The bounds of \citet{mohri2016accelerating} for Setting~\ref{senarios:psi-t-revealed} correspond to the non-proximal FTRL case.
As such, one has to set the step-size sequences according to the Dual-Averaging AdaGrad recipe (c.f. \cref{tbl:ftrl-special-cases}), which requires an additional regularization of $\tl{q}_0 = \sqrt{\delta} \|\cdot\|_2^2$.
In contrast, in \FTProx, $\tl{q}_0=0$.
This $\delta$ value makes Dual-Averaging AdaGrad non-scale-free, while \FTProx is scale-free (i.e., the $x_t$ are independent of the scaling of $f_t$).
Our analysis avoids this problem by the early separation of the proximal ($p_t$) and non-proximal  regularizers ($q_t$) in \AdaFTRL.
In particular, $p_t, \tl{q}_t$ in \cref{cor:composite-ao-ftrl} can be set as
$\tl{q}_t=0$ and $p_t = \frac{\eta_t - \eta_{t-1}}{2} \| x-x_t \|^2$ with $\eta_t = \eta \sqrt{\sum_{s=1}^{t} \|g_s-\tl{g}_s\|_*^2},
\eta > 0$ for $ t=1,2,\dots,T$.
This gives composite-objective AO-FTRL-Prox, a scale-free adaptive optimistic algorithm for Setting~\ref{senarios:psi-t-revealed}.

In addition, using \cref{thm:chiang-smooth-bound}, we can obtain a variational bound for composite-objective optimistic \FTProx (proved in  \cref{apx:optimistic-learning-proofs}), which was not available through the analysis of \citet{mohri2016accelerating} even under Setting~\ref{senarios:psi-t-known}:
\begin{corollary}
  \label{cor:the-final-attack}
  Let $\psi_t, t=1,2,\dots,T$, be convex and satisfy the conditions of \cref{cor:composite-obj-hidden-psi}. Further assume that $f_t$ are convex and $L$-smooth w.r.t. some norm $\|\cdot\|$.
  Suppose that $\cX$ is closed, and let $R^2 = \sup_{x,y \in \cX} \|x-y\|^2 < + \infty$ be the diameter of $\cX$ measured in $\| \cdot \|$. Define $\eta = 2/R$.
  Suppose we run composite-objective AO-FTRL (update \eqref{eq:composite-ao-ftrl-update}) with the following parameters: $\tl{q}_0 = 0$, and for $t=1,2,\dots,T$, $\tl{g}_{t} = g_{t-1}$, $\tl{q}_t=0$, and $p_t = \frac{\eta_t - \eta_{t-1}}{2} \| x-x_t \|^2$, where $\eta_0 = 0$ and $\eta_t = 4 R L^2 + \eta \sqrt{\sum_{s=1}^{t} \|g_s - \tl{g}_s\|_{*}^2}$ for $t \ge 1$.
  Then,
  \begin{align}
    R^{(\ell)}_T(x^*) \le &\ 2 R^3 L^2 + R + 2 R \sqrt{2 D_{\| \cdot \|}} = O \left( R \sqrt{D_{\|\cdot\|}} \right)\,.
    \label{eq:the-final-attack}
  \end{align}
\end{corollary}
Note that the learning rate $\eta_t$ is bounded from below (by $4RL^2$), which is essential in the algorithm to achieve a combination of the best properties of \citet{mohri2016accelerating}, \citet{chiang2012online}, and \citet{rakhlin2013online,rakhlin2013optimization}:
First, like \citet{mohri2016accelerating}, we allow the use of composite-objectives.
Second, similarly to \citet{chiang2012online} (but unlike \citealt{mohri2016accelerating,rakhlin2013online,rakhlin2013optimization}) our bound applies to the variation of general convex smooth functions, and is still optimal when $L=0$ (e.g., Corollary 2 of \citealt{rakhlin2013optimization}).
Third, we do not need the knowledge of $D_{\|\cdot\|}$ (required by \citealt{chiang2012online}) to set the step-sizes, and avoid the regret penalty of using a doubling trick (as done by \citet{rakhlin2013online}).
Fourth, in the practically interesting case of a composite L1 penalty ($\psi_t = \alpha_t \|\cdot\|_{1}$), \FTProx, which is the basis of our algorithm, gives sparser solutions \citep{mcmahan2014survey} than MD, which is the basis of the algorithms of \citet{chiang2012online} and \citet{rakhlin2013optimization}.
Fifth, when $L=0$, the algorithm is scale-free (unlike \citealt{mohri2016accelerating} and \citealt{rakhlin2013online}).
Finally, the results apply to the variation measured in any norm.

 \section{Application to non-convex optimization}\label{sec:non-convex}
In this section, we collect some results related to applying online optimization methods to non-convex optimization problems.
This is another setting where the strength of our derivations is apparent: As we shall see, without any extra work, the results imply and extend previous results.

Central to this extension is the decomposition of assumptions in our analysis: we are not using the convexity of $f_t$ in \cref{lem:regret-decomposition} or \cref{thm:cheating-regret}, but only at the very last stage of the analysis, where convexity can ensure that \cref{assum:positive-bregman-ft} holds.
Thus, the analysis easily extends to non-convex optimization problems where \cref{assum:positive-bregman-ft} either holds or could be replaced by another technique at the final stage of the analysis. In the rest of this section, we explore such classes of non-convex problems,
which are also related to the Polyak-\L{}ojasiewicz (PL) condition used in the non-convex optimization and learning community.
For background and a summary of related work, consult \citet{KaNoSch16}.

\subsection{Stochastic optimization of star-convex functions}
\label{sec:star-convex}
First, we explore the class of non-convex functions for which \cref{assum:positive-bregman-ft} directly holds.
As it turns out, this is a much larger class of functions than convex functions.
In particular, consider the so-called ``star-convex'' functions \citep{nesterov2006cubic}:\footnote{We modify the definition so that it is relative to a given fixed global minimizer as this way we capture a larger class of functions and this is all we need.}
 \begin{definition}[Star-convex function]
   A function $f$ is \emph{star-convex} at a point $x^*$ if and only if $x^*$ is a global minimizer of $f$, and for all $\alpha \in [0,1]$ and all $x \in \dom(f)$:
   \begin{align}
     f(\alpha x^* + (1-\alpha) x) \le \alpha f(x^*) + (1-\alpha) f(x)\,.
     \label{eq:star-convex-def}
   \end{align}
   A function is said to be star-convex when it is star-convex at some of its global minimizers.
 \end{definition}
The name ``star-convex'' comes from the fact that the sub-level sets $L_\beta = \{ x \,:\, f(x) \le \beta \}$ of a function $f$ that is star-convex at $x^*$ are star-shaped with center $x^*$ (recall that a set $U$ is star-convex with center $x$ if for any $y\in U$, the segment between $x$ and $y$ is included in $U$). However, note that there are functions whose sub-level sets are star-convex that are themselves not star-convex. In particular, functions that are increasing along the rays (IAR) started from their global minima have star-shaped sub-level sets and vice versa, but some of these functions (e.g., $f(x) = \sqrt{|x|}$, $x\in \R$) is clearly not star-convex.
Recall that quasi-convex functions are those whose sub-level sets are convex. In one dimension a star-convex function
is thus also necessarily quasi-convex. However, clearly, there are star-convex functions (such as $x\mapsto |x|\one{|x|\le 1} + 2|x| \one{|x|>1}$, $x\in \R$) that are not convex and in multiple dimensions there are star-convex functions that are not quasi-convex (e.g., $x \mapsto \norm{x}^2 g(\frac{x}{\norm{x}_2^2})$ where $g(u)$ is, say, the sine of the angle of $u$ with the unit vector $e_1$).

Star-convex functions arise in various optimization settings, often related to
sums of squares \citep{nesterov2006cubic,LeeValiant2016:FOCS}.
It is easy to see from the definitions that the set of star-convex functions is closed under nonsingular
affine domain transformations, addition (of functions having the same center)
and multiplication by nonnegative constants. Further, for $x\in \R^d$,
$x \mapsto \prod_i |x_i|^{p_i}$ is star-convex (at zero) whenever $\sum_i p_i\ge 1$.
For further properties and examples see \citet{LeeValiant2016:FOCS}.

We can immediately see that Assumption~\ref{assum:positive-bregman-ft} holds for star-convex functions:
\begin{lemma}[Non-negative Bregman divergence for star-convex functions]
\label{lem:bregstar}
Let $f$ be a directionally differentiable function with global optimum $x^*$.
Then, $f$ is star-convex at $x^*$ if and only if for all $x \in \cH$,
\begin{align*}
  \cB_{f}(x^*, x) \ge 0\,.
\end{align*}
\end{lemma}
\begin{proof}
Both directions are routine. For illustration we provide the proof of the forward direction.
Assume without loss of generality that $x^*=0$ and $f(x^*)=0$.
Then star-convexity at $x^*$ is equivalent to having $f(\alpha x) \le \alpha f(x)$ for any $x$ and $\alpha \in [0,1]$.
Further, $\cB_{f}(x^*, x) \ge 0$ is equivalent to
$-f(x) - f'(x;-x)\ge 0$. Now, $f'(x;-x) = \lim_{\alpha\downarrow 0} \frac{f( x+ \alpha(-x) ) - f(x)}{\alpha}$.
Under star-convexity, $f(x+\alpha(-x)) = f((1-\alpha)x) \le (1-\alpha) f(x)$. Hence,
$f'(x;-x) \le \lim_{\alpha\downarrow 0} \frac{ (1-\alpha) f(x) - f(x)}{\alpha} = - f(x)$.
\end{proof}
Thus, \cref{cor:recoveries-stoch,cor:recoveries-stoch-smooth} apply to star-convex functions.
In other words:
\begin{itemize}
  \item For stochastic optimization of directionally-differentiable star-convex functions in Hilbert spaces, \AdaFTRL and \AdaMD converge to the global optimum \emph{with the same rate as they converge for convex functions} (including fast rates due to other assumptions, e.g., smoothness).
\end{itemize}
Of course, a similar result holds for the online setting, too, but in this case the assumption that each $f_t$ is star-convex w.r.t. the same center $x^*$ becomes restrictive.
\begin{remark}
Since the rate of regret depends on the norm of the gradients $g_t$, to get fast rates one needs to control these norms.
This is trivial if $f$ is Lipschitz-continuous. However, some star-convex functions are not Lipschitz, even arbitrarily close
to the optima (e.g., $f(x,y) = (\sqrt{|x|} + \sqrt{|y|})^2$). For such functions, \citet{LeeValiant2016:FOCS} propose alternative methods to gradient descent. However, it seems possible to control the norms in these settings
using additional regularization (as in the normalized gradient descent method); see, e.g., the work of \citealt{HaLeSS15}, and the recent work of \citet{levy2017online}.
Exploring this idea is left for future work.
\end{remark}

\subsection{Beyond star-convex functions}
\label{sec:tau-star-convex}
Inspecting our proofs we may notice that \cref{assum:positive-bregman-ft} is unnecessarily restrictive: to maintain the same rate of growth for regret, it suffices for the sum of Bregman divergences to grow with the same rate as the rest of the bound, rather than being negative and hence dropped. This extends all of our results to another interesting class of non-convex functions which generalize star-convexity:
\begin{definition}[$\tau$-star-convexity, \citet{hardt2016gradient}]
\label{def:taustarconv}
Let $f$ be
  a directionally differentiable function $f$ with global optimum $x^*$.
Then $f$ is \emph{$\tau$-star-convex}\footnote{
\citet{hardt2016gradient} define the same concept under $\tau$-weakly-quasi-convexity.
However, per our previous discussion, it appears more appropriate to call this property $\tau$-star-convexity.
Especially since when $\tau=1$ we get back star-convexity, which, as we have seen is not a weakening of quasi-convexity.
}
 on a set $\cX$ at
  $x^* \in \cH$ if there is $\tau > 0$ such that for all $x \in \cX \cap \dom(f)$,
  \begin{align}
     \tau (f(x)-f(x^*))\le -f'(x; x^* - x)\,.
    \label{eq:tau-quasi-convex-def}
  \end{align}
\end{definition}
Note that by \cref{lem:bregstar}, star-convexity corresponds to the case when $\tau = 1$.
\citet{hardt2016gradient} demonstrated that an objective function that arises naturally in the identification of certain
class of linear systems is $\tau$-star-convex with some $\tau>0$. For differentiable functions, \eqref{eq:tau-quasi-convex-def}
is equivalent to $f(x)-f(x^*) \le \frac{1}{\tau} \dotx{\nabla f(x),x-x^*}$, so it is a simple generalization of the linear upper bound one
typically uses to reduce online convex optimization to online linear optimization. Therefore, any regret bound that is proved via upper bounding
linearized losses automatically extends to $\tau$-star-convex functions. However, in general, it may require substantial work to identify what assumptions are used exactly in
proving an upper bound on the linearized loss (e.g., \citealp{hardt2016gradient} reproved the convergence guarantees for smooth SGD). The next lemma
shows that our techniques can automatically separate the effects of different assumptions and provide fast regret rates under appropriate circumstances.

\begin{lemma}[Basic regret bound under $\tau$-star-convexity]
\label{lem:rboundtaustar}
  Let $f$ be locally directionally differentiable and $\tau$-star-convex on a set $\cX$ at $x^*$, $f_1\dots=f_T = f$.
  Then, for all $x_t \in \cX \cap \dom(f_t)$ and $g_t \in \cH$ ($t=1,2, \dots, T$),
  \begin{align*}
    R_T(x^*) \le \frac{1}{\tau} \left(R_T^{+} (x^*) + \sum_{t=1}^{T} \dotx{g_t, x_{t} - x_{t+1}} + \delta_t \right)\,.
  \end{align*}
\end{lemma}
\begin{proof}
The proof can be derived from the right-hand side of \eqref{eq:regret-decomposition}, but a shorter direct proof is also available:
Add and subtract $\dotx{g_t,x^*-x_t}$ to the right-hand side of \eqref{eq:tau-quasi-convex-def}.
Noticing that $-f'(x_t;x^*-x_t) + \dotx{g_t,x^*-x_t} = \delta_t$, summing up and using the definition $R_T^+(x^*) = \sum_t \dotx{g_t,x_{t+1}-x^*}$ gives the result.
\end{proof}
Now since the regret was bounded through the expression in the parentheses of the previous display,
\cref{cor:recoveries-stoch,cor:recoveries-stoch-smooth} apply. In other words:
\begin{itemize}
  \item For stochastic optimization of directionally-differentiable $\tau$-star-convex functions in Hilbert spaces, \AdaFTRL and \AdaMD enjoy \emph{$1/\tau$-times the same regret as when they are applied to linearized loss functions} (including fast rates due to other assumptions, e.g., smoothness).
\end{itemize}

In the convex case the strong convexity of the losses (\cref{assum:strong-cvx-loss}) implied that their Bregman divergences are nonnegative (\cref{assum:positive-bregman-ft}).
The natural generalization of this leads to the following definition:
\begin{definition}[$\tau$-star-strong-convexity]
Let $f$, $r$ be directionally differentiable and let $x^*$ be a global minimum of $f$.
Then, $f$ is $\tau$-star-strongly-convex w.r.t. $r$ if $S:=\dom(f) \cap \dom(r)$ is non-empty
and there exists $\tau > 0$ such that for all $x \in S$ and some minimizer $x^*$ of $f$,
\begin{align}
\tau(f(x) - f(x^*)) \le -f'(x;x^*-x)- \cB_{r}(x^*, x)\,.
  \label{eq:taussc}
\end{align}
\end{definition}

Replacing $\tau$-star-convexity with $\tau$-star-strong-convexity gives the following analogue of
\cref{lem:rboundtaustar}:
\begin{lemma}[Basic regret bound under $\tau$-star-strong-convexity]
\label{lem:rboundtaustarstrong}
  Let $f$, $r$ be locally directionally differentiable.
  Assume that $f$ is $\tau$-star-strongly-convex w.r.t. $r$ at $x^*$ on a set $\cX$.
  Then, for all $x_t \in \cX \cap \dom(f_t)\cap \dom(r)$ and $g_t \in \cH$ ($t=1,2, \dots, T$),
  \begin{align*}
    R_T(x^*) \le \frac{1}{\tau} \left(R_T^{+} (x^*) + \sum_{t=1}^{T} \dotx{g_t, x_{t} - x_{t+1}} + \delta_t - \cB_{r}(x^*, x_t) \right)\,.
  \end{align*}
\end{lemma}
\begin{proof}
The proof follows the same step as that of \cref{lem:rboundtaustar}, except that we need to use \eqref{eq:taussc} instead of \eqref{eq:tau-quasi-convex-def}.
\end{proof}
It follows that the same manipulations as in  \cref{cor:recoveries-stoch,cor:recoveries-stoch-smooth}  imply:
\begin{itemize}
  \item For stochastic optimization of directionally-differentiable $\tau$-star-strongly-convex functions in Hilbert spaces, \AdaFTRL and \AdaMD converge to the global optimum \emph{with $1/\tau$-times the same rate as they converge for strongly convex functions}.
\end{itemize}

It appears that $\tau$-star-strong-convexity is related to the Polyak-\L{}ojasiewicz (PL) inequality. Recall that a differentiable function $f$ satisfies the PL inequality with constant $\mu>0$ if
\begin{align*}
 \mu (f(x) - f(x^*)) \le \frac12 \norm{\nabla f(x)}_2^2\,,
\end{align*}
where $x^*$ is a global minimizer of $f$. Proposed independently and simultaneously by \cite{Poly63} and \cite{Loj63}, the PL inequality appears to play a fundamental role in the study of incremental gradient algorithms (see \citealt{KaNoSch16} and the references therein). As star-convexity, the PL inequality can also be satisfied by non-convex functions, partly explaining the prominent role it plays in the analysis of gradient methods. We can see that $\tau$-star-strong-convexity implies the PL inequality when $r$ is the squared Euclidean norm:
\begin{lemma}[PL is implied by star-strong-convexity]
Let $r(x) = \frac12 \norm{x}_2^2$ and let $f$ be differentiable.
If $f$ is $\tau$-star-strongly-convex w.r.t. $r$, then $f$ also satisfies the PL inequality with $\mu = \tau$.
\end{lemma}
\begin{proof}
Assume that $f$ satisfies \eqref{eq:taussc}.
We have
\[
-f'(x;x-x^*) = \dotx{ \nabla f(x), x^*-x} \le \frac12 \left( \norm{\nabla f(x)}^2 + \norm{x^*-x}^2 \right)\,,
\]
where the second step follows from the Fenchel-Young inequality. As it is well known, $B_r(x,y) = \frac12 \norm{x-y}^2$.
Thus,  \eqref{eq:taussc} implies that $\tau(f(x) - f(x^*)) \le  \frac12 \norm{\nabla f(x)}_2^2 $.
\end{proof}

Finally, note that the results above do not preclude the use of other algorithmic ideas, such as implicit-update, non-linearized, or composite-objective learning; the same extensions of \cref{cor:recoveries-stoch,cor:recoveries-stoch-smooth}, as discussed in \cref{sec:composite-objective,sec:implicit-update}, apply here as well. In addition, there are interesting classes of non-convex problems other than the PL class; see, e.g., \citet{KaNoSch16}.
A direction for future work is to explore whether these classes relate to specific conditions on Bregman divergences, and whether similar convergence results for general adaptive optimization are also possible under these function classes.

\section{Discussion}
\label{sec:related-work}
In this section we compare the results obtained in this paper to the previous attempts at unified analysis of adaptive FTRL and MD. A starting point of our work was the unifying treatment of online learning algorithms by \citet{mcmahan2014survey}, as well as the generalized adaptive FTRL analysis of \citet{orabona2015generalized}.

\subsection{Comparison to the analysis of \citet{mcmahan2014survey}}
\label{sec:mcmahan}
\citet{mcmahan2014survey} also studied a unified, modular analysis of MD and FTRL algorithms (albeit with different modules), assuming that the regularizers $p_t,q_t,r_t$ are convex, non-negative, and satisfy \cref{assum:strong-cvx-regularizer}. \AdaFTRL and \AdaMD encompass all of the algorithms they considered. In particular, their Theorems 1 and 2 are special cases of Corollary~\ref{cor:recoveries} (recall that non-linearized FTRL, and in particular strongly-convex FTRL, are also special cases of \AdaFTRL; see \cref{sec:implicit-update}). In addition, our analysis applies more generally to infinite-dimensional Hilbert spaces, our presentation of \AdaFTRL encompasses a larger set of algorithms, the relaxed assumptions under which we analyzed \AdaFTRL and \AdaMD remove certain practical limitations that existed in the work of \citet{mcmahan2014survey}, and our analysis captures a wider range of learning settings. We discuss these improvements below.

Importantly, \citet{mcmahan2014survey} also provides a reduction from MD to a version of \FTProx. This, in particular, illuminates important differences between MD and FTRL in composite-objective learning. We refer the reader to Section~6 of their paper. We decided to keep the presentation of the two algorithms separate to facilitate the relaxation of the assumptions on the regularizers; see \cref{assum:ftrl,assum:md} and the discussion below.

\subsubsection{Relaxing the assumptions on the regularizers}
A central part of the modularity of our analysis comes from the flexibility of \cref{assum:ftrl,assum:md} on the regularizers of \AdaFTRL and \AdaMD.
In particular, unlike \citet{mcmahan2014survey}, we do not assume that the individual regularizers $p_t,q_t,r_t$ are non-negative or convex.

This relaxation provides two benefits. First, with the non-negativity restriction removed, we can add arbitrary, possibly linear, components to the regularizers. As we showed above, this resulted in a simple recovery and analysis of optimistic FTRL and a new class of optimistic MD algorithms (\cref{sec:optimistic-learning}), as well as a straightforward recovery of implicit and non-linearized updates, even for non-convex functions (\cref{sec:implicit-update}).

Second, with the convexity assumption removed, \AdaFTRL and \AdaMD can accommodate algorithmic ideas such as non-decreasing regularization.
For example, AdaDelay \citep{sra2016adadelay}, an instance of \AdaMD, uses $r_{1:t} = \eta_t \|\cdot\|^2$, but $\eta_t$ is not guaranteed to be non-decreasing, i.e., $r_t$ could be negative and non-convex (while $r_{0:t}$ still remains convex for all $t$). Now, note that MD and \FTProx are closely related. Particularly, if $p_t$ are themselves Bregman divergences (as in proximal \AdaGrad), then \FTProx and MD have identical regret bounds. Therefore, the techniques of \citet{sra2016adadelay} for controlling the regularizer terms in the bound could be naturally applied, almost with no modification, to an \FTProx version of AdaDelay. This extension to \FTProx is interesting since, as mentioned before, composite \FTProx with an L1 penalty tends to produce sparser results compared to \AdaMD \citep[Section 6]{mcmahan2014survey}. Thus, while this variant of \FTProx is a special case of \AdaFTRL (e.g., \cref{cor:recoveries} applies), it was not clear how to analyze this algorithm under the assumptions made by \citet{mcmahan2014survey}.

Finally, the choice to separate the proximal and non-proximal regularizers in \AdaFTRL provides certain conveniences. In particular, the $q_t$ terms can take the role of incorporating information (such as composite terms) into \AdaFTRL, while the proximal part $p_t$ remains intact. This precludes the need to provide a separate analysis for \FTProx every time the structure of information changes (e.g., when implicit updates are added). Thus, unlike Section 5 of \citet{mcmahan2014survey}, we did not need to provide a separate analysis (their Theorem 10) for composite-objective \FTProx. We also avoided the complications with composite optimistic \FTProx as in \citet{mohri2016accelerating}; see Section~\ref{sec:optimistic-learning}.

\subsubsection{The regret decomposition and analysis of new learning settings}
\label{sec:flexible-settings}
In comparison to \citet{mcmahan2014survey}, the analysis we provided exhibits much flexibility across learning settings.
In particular, the regret decomposition given by Lemma~\ref{lem:regret-decomposition} enabled us to accommodate a wide range of learning settings, and separate the effect of the learning setting from the forward regret of the algorithm. Building on this, for example, we provided a clean analysis of variational and variance-dependent bounds for smooth losses (and generalized them to adaptive algorithms). In addition, by encapsulating the effect of loss convexity into \cref{assum:positive-bregman-ft}, we could generalize the analysis to certain non-convex classes.

\subsection{Comparison to the analysis of \citet{orabona2015generalized}}
\label{sec:orabona}
\citet{orabona2015generalized} study a special case of \AdaFTRL, where $p_t \equiv 0$ and Assumption~\ref{assum:strong-cvx-regularizer} holds. The main result of \citet{orabona2015generalized}, i.e., their Lemma~1, can be though of as playing the same role as \eqref{eq:put-together-ftrl}. We emphasize, however, that their Lemma~1 is a quite general result. For example, with a few algebraic operations we could recover a special case of \cref{thm:cheating-regret} from their Lemma~1, by setting $z_t=0$ and moving the linear components to their $f_t$ functions. Nevertheless, our analysis extends the work of \citet{orabona2015generalized} to infinite-dimensional Hilbert-spaces, to \FTProx algorithms, and to \AdaMD. Furthermore, we demonstrated a principled way of mixing algorithmic ideas and incorporating information from the learning setting into FTRL and MD using the $q_t$ functions. Finally, the comments of Section~\ref{sec:flexible-settings} apply.

Importantly, the authors also provide a compact analysis of the Vovk-Azoury-Warmuth algorithm, as well as online binary classification algorithms. These results are essentially obtained from combining their Lemma~1 with interesting regret decompositions other than the one we presented in Lemma~\ref{lem:regret-decomposition}. It seems possible to combine their regret decompositions with our analysis to extend their result to \AdaMD algorithms, and to obtain refined bounds for smooth losses. We leave this direction for future work.

\section{Conclusion and future work}
We provided a generalized, unified and modular framework for analyzing online and stochastic optimization algorithms, and demonstrated its flexibility on several existing, as well as new, algorithms and learning settings. Our framework can be used together with other algorithmic ideas and learning settings, e.g., adaptive delayed-feedback algorithms like AdaDelay~\citep{sra2016adadelay}, but results related to this where out of the scope of this work.
There are many interesting questions related to non-convex optimization; while we showed that our results extend to the so-called $\tau$-star(-strongly)-convex functions, which have already found some applications, it remains to be seen whether they also extend to other settings, such as optimization of quasi-convex functions, or functions that satisfy the Polyak-\L{}ojasiewicz inequality.
Exploring these and other applications of this framework is left for future work.

\acks{
We would like to thank the anonymous reviewers for several insightful comments that helped improve the quality of the paper.
This work was supported by Amii (formerly AICML) and NSERC.
}

\bibliography{refs}

\newpage
\appendix

\section{Proof of the regret decomposition Lemma~\ref{lem:regret-decomposition}}
\label{app:prL1}
\begin{proof}
By definition,
\begin{align*}
  f_t(x_t) - f_t(x^*) & = -\cB_{f_t}(x^*, x_t) - f'(x_t; x^*-x_t) \\
  & = \dotx{g_t, x_t - x^*} - \cB_{f_t}(x^*, x_t) + \delta_t \\
  & =  \dotx{g_t, x_{t+1} - x^*} + \dotx{g_t, x_t - x_{t+1}} - \cB_{f_t}(x^*, x_t) + \delta_t\,.
\end{align*}
Summing over $t$ completes the proof.
\end{proof}

\section{Formal statements and proofs for the standard results described in Table~\ref{tbl:summary-of-results}}
\label{sec:standard_proofs}

Putting Lemma~\ref{lem:regret-decomposition} and Theorem~\ref{thm:cheating-regret} together, for \AdaFTRL we obtain
\begin{align}
  R_T(x^*) \le &\ - \sum_{t=1}^{T} \cB_{f_t}(x^*, x_{t}) + \sum_{t=0}^{T} \left( q_t(x^*) - q_{t}(x_{t+1}) \right) + \sum_{t=1}^T \left( p_t(x^*) - p_t(x_t) \right) \nonumber \\
  &\ - \sum_{t=1}^{T} \cB_{r_{1:t}}(x_{t+1}, x_t) + \sum_{t=1}^{T} \dotx{g_t, x_{t} - x_{t+1}} + \sum_{t=1}^{T} \delta_t \,, \label{eq:put-together-ftrl}
\intertext{whereas for \AdaMD,}
  R_T(x^*) \le &\ - \sum_{t=1}^{T} \cB_{f_t}(x^*, x_{t}) + \sum_{t=0}^{T} \left( q_t(x^*) - q_{t}(x_{t+1}) \right) + \sum_{t=1}^{T} \cB_{p_t}(x^*, x_t) \nonumber \\
  &\
  - \sum_{t=1}^{T} \cB_{r_{1:t}}(x_{t+1}, x_t) + \sum_{t=1}^{T} \dotx{g_t, x_{t} - x_{t+1}}  + \sum_{t=1}^{T} \delta_t \,. \label{eq:put-together-md}
\end{align}

Next we prove the concrete regret bounds, given in Table~\ref{tbl:summary-of-results}, based on the above. A schematic view of the proof ideas is given in Figure~\ref{fig:proofs}.

\begin{corollary}\label{cor:recoveries}
  Consider the ``Online Optimization'' setting (Assumption \ref{assum:online-opt}), using \AdaFTRL (under Assumption~\ref{assum:ftrl}) or \AdaMD (under Assumption~\ref{assum:md}). Suppose that Assumptions~\ref{assum:positive-bregman-ft} and ~\ref{assum:strong-cvx-regularizer} hold. Then,
  \begin{enumerate}[(i)]
    \item \label{itm:simple-online-md}
    the regret of \AdaMD is bounded as
    \begin{align*}
      R_T(x^*) \le &\ \sum_{t=0}^{T} \left( q_t(x^*) - q_{t}(x_{t+1}) \right) + \sum_{t=1}^{T} \cB_{p_t}(x^*, x_t)  + \sum_{t=1}^{T} \frac{1}{2} \|g_t\|_{(t),*}^2\,;
    \end{align*}
    \item \label{itm:simple-online} the regret of \AdaFTRL is bounded as
    \begin{align*}
      R_T(x^*) \le &\ \sum_{t=0}^{T} \left( q_t(x^*) - q_{t}(x_{t+1}) \right) + \sum_{t=1}^{T} \left( p_t(x^*) - p_t(x_t) \right) + \sum_{t=1}^{T} \frac{1}{2} \|g_t\|_{(t),*}^2\,;
    \end{align*}
    \item \label{itm:stcvx-online-md}
    under Assumption~\ref{assum:strong-cvx-loss}, the regret of \AdaMD is bounded as
    \begin{align*}
      R_T(x^*) \le &\ \sum_{t=0}^{T} \left( q_t(x^*) - q_{t}(x_{t+1}) \right) + \sum_{t=1}^{T} \frac{1}{2} \|g_t\|_{(t),*}^2\,.
    \end{align*}
  \end{enumerate}
 \end{corollary}
 \begin{proof}
   Note that by Assumption~\ref{assum:online-opt}, we have
   \begin{align}
     \delta_t \le 0\,, \label{eq:delta-online}
   \end{align}
   for all $t=1,2,\dots,T$. In addition, by the Fenchel-Young inequality and \cref{assum:strong-cvx-regularizer},
   \begin{align}
     \dotx{g_t, x_t - x_{t+1}} \le
     &\
     \frac{1}{2}\|x_{t} - x_{t+1}\|_{(t)}^2 + \frac{1}{2} \| g_t \|_{(t),*}^2
     \nonumber
     \\
     \le &\
     \cB_{r_{1:t}}(x_{t+1}, x_t) + \frac{1}{2} \| g_t \|_{(t),*}^2 \,.\label{eq:fenchel-young}
   \end{align}
   Putting \eqref{eq:delta-online}, \eqref{eq:fenchel-young}, and Assumption~\ref{assum:positive-bregman-ft} into \eqref{eq:put-together-ftrl} and \eqref{eq:put-together-md} and cancelling out the matching terms proves \eqref{itm:simple-online-md} and \eqref{itm:simple-online}.
   Finally, to prove \eqref{itm:stcvx-smooth-stoch-md}, we use Assumption~\ref{assum:strong-cvx-loss} to cancel the $\cB_{f_t}(x^*,x_t)$ terms with the $\cB_{p_t}(x^*,x_t)$ terms (rather than dropping them by \cref{assum:positive-bregman-ft}).
 \end{proof}

 \begin{corollary}\label{cor:recoveries-stoch}
  Consider the ``Stochastic Optimization'' setting (Assumption~\ref{assum:stoch-opt}), using \AdaFTRL (under Assumption~\ref{assum:ftrl}) or \AdaMD (under Assumption~\ref{assum:md}). Suppose that Assumptions~\ref{assum:positive-bregman-ft} and ~\ref{assum:strong-cvx-regularizer} hold. Then,
  \begin{enumerate}[(i)]
    \item \label{itm:simple-stoch-md}
    the regret of \AdaMD is bounded as
    \begin{align*}
      \E{R_T(x^*)} \le &\ \E{\sum_{t=0}^{T} \left( q_t(x^*) - q_{t}(x_{t+1}) \right) + \sum_{t=1}^{T} \cB_{p_t}(x^*, x_t)  + \sum_{t=1}^{T} \frac{1}{2} \|g_t\|_{(t),*}^2}\,;
    \end{align*}
    \item \label{itm:simple-stoch}
    the regret of \AdaFTRL is bounded as
    \begin{align*}
      \E{R_T(x^*)} \le &\ \E{\sum_{t=0}^{T} \left( q_t(x^*) - q_{t}(x_{t+1}) \right) + \sum_{t=1}^{T} \left( p_t(x^*) - p_t(x_t) \right) + \sum_{t=1}^{T} \frac{1}{2} \|g_t\|_{(t),*}^2}\,;
    \end{align*}
    \item \label{itm:stcvx-stoch-md}
    under Assumption~\ref{assum:strong-cvx-loss}, the regret of \AdaMD is bounded as
    \begin{align*}
      \E{R_T(x^*)} \le &\ \E{\sum_{t=0}^{T} \left( q_t(x^*) - q_{t}(x_{t+1}) \right) + \sum_{t=1}^{T} \frac{1}{2} \|g_t\|_{(t),*}^2}\,.
    \end{align*}
  \end{enumerate}
\end{corollary}
\begin{proof}
  Let $f_t = f$ in Lemma~\ref{lem:regret-decomposition} (hence in \eqref{eq:put-together-ftrl} and \eqref{eq:put-together-md}), and note that by Assumption~\ref{assum:stoch-opt}, we have
  \begin{align}
    \E{\delta_t} = \E{ f'(x_t; x^* - x_t) - \dotx{\E{g_t | x_t}, x_t - x^*} } \le 0\,, \label{eq:delta-stoch}
  \end{align}
  for all $t=1,2,\dots,T$.
  Similar to the proof of Corollary~\ref{cor:recoveries}, putting \eqref{eq:delta-stoch}, \eqref{eq:fenchel-young}, and Assumption~\ref{assum:positive-bregman-ft} into \eqref{eq:put-together-ftrl} and \eqref{eq:put-together-md} proves \eqref{itm:simple-stoch-md} and \eqref{itm:simple-stoch}. Finally, to prove \eqref{itm:stcvx-smooth-stoch-md}, one can use Assumption~\ref{assum:strong-cvx-loss} to cancel the $\cB_{f}(x^*,x_t)$ terms with the $\cB_{p_t}(x^*,x_t)$ terms (rather than dropping them by \cref{assum:positive-bregman-ft}).
\end{proof}

\begin{corollary}\label{cor:recoveries-stoch-smooth}
 Consider the ``Stochastic Optimization'' setting (Assumption~\ref{assum:stoch-opt}), using \AdaFTRL (under Assumption~\ref{assum:ftrl}) or \AdaMD (under Assumption~\ref{assum:md}).
 Suppose that Assumptions~\ref{assum:positive-bregman-ft}, ~\ref{assum:smooth-f} hold,
 and Assumption~\ref{assum:strong-cvx-regularizer} holds with $r_{1:t}-\|\cdot\|^2/2$ in place of $r_{1:t}$.\footnote{The modification to Assumption~\ref{assum:strong-cvx-regularizer} is equivalent to adding an extra $\|x\|^2/2$ regularizer to \AdaFTRL and \AdaMD.}
 Let $f^* := \inf_{x \in \cX} f(x)$, and define $D := f(x_1) - f^*$ and $\sigma_t := g_t - \nabla f(x_t)$. Then,
 \begin{enumerate}[(i)]
   \item \label{itm:smooth-stoch-md}
   the regret of \AdaMD is bounded as
   \begin{align*}
     &\ \E{R_T(x^*)} \le
     \\
     &\ \E{\frac{1}{2} \|x^*-x_1\|^2 + \sum_{t=0}^{T} \left( q_t(x^*) - q_{t}(x_{t+1}) \right) + \sum_{t=1}^{T} \cB_{p_t}(x^*, x_t) + \sum_{t=1}^{T} \frac{1}{2} \|\sigma_t\|_{(t),*}^2 + D}\,;
   \end{align*}
   \item \label{itm:smooth-stoch}
   the regret of \AdaFTRL is bounded as
   \begin{align*}
     &\ \E{R_T(x^*)} \le
     \\
     &\ \E{\frac{1}{2} \|x^*\|^2 + \sum_{t=0}^{T} \left( q_t(x^*) - q_{t}(x_{t+1}) \right) + \sum_{t=1}^{T} \left( p_t(x^*) - p_t(x_t) \right) + \sum_{t=1}^{T} \frac{1}{2} \|\sigma_t\|_{(t),*}^2 + D}\,;
   \end{align*}
   \item \label{itm:stcvx-smooth-stoch-md}
   under Assumption~\ref{assum:strong-cvx-loss}, the regret of \AdaMD is bounded as
   \begin{align*}
     \E{R_T(x^*)} \le &\ \E{\frac{1}{2} \|x^* - x_1\|^2 + \sum_{t=0}^{T} \left( q_t(x^*) - q_{t}(x_{t+1}) \right) + \sum_{t=1}^{T} \frac{1}{2} \|\sigma_t\|_{(t),*}^2 + D}\,.
   \end{align*}
 \end{enumerate}
\end{corollary}
\begin{proof}
  Note that for all $t=1,2,\dots,T$, by Assumption~\ref{assum:smooth-f} and the Fenchel-Young inequality,
  \begin{align}
    \dotx{g_t, x_t - x_{t+1}} &= f(x_{t}) - f(x_{t+1}) + \cB_{f}(x_{t+1}, x_t) + \dotx{g_t - \nabla f(x_t), x_t - x_{t+1}} \nonumber\\
    & \le f(x_{t}) - f(x_{t+1}) + \frac{1}{2} \|x_t - x_{t+1}\|^2 + \dotx{\sigma_t, x_t - x_{t+1}} \nonumber\\
    & \le f(x_{t}) - f(x_{t+1}) + \frac{1}{2} \|x_t - x_{t+1}\|^2 + \frac{1}{2} \|\sigma_t\|_{(t), *}^2 + \frac{1}{2} \|x_t - x_{t+1}\|_{(t)}^2   \label{eq:smooth-linear-term}
  \end{align}
  Putting \eqref{eq:delta-stoch}, \eqref{eq:smooth-linear-term}, and Assumption~\ref{assum:positive-bregman-ft} into \eqref{eq:put-together-ftrl} and \eqref{eq:put-together-md}, telescoping the $f$ terms, using $f(x_{T+1}) \ge f^*$, and canceling out the matching terms gives \eqref{itm:smooth-stoch-md} and \eqref{itm:smooth-stoch}.
  Finally, to prove \eqref{itm:stcvx-smooth-stoch-md}, one can use Assumption~\ref{assum:strong-cvx-loss} to cancel the $\cB_{f}(x^*,x_t)$ terms with the $\cB_{p_t}(x^*,x_t)$ terms (rather than dropping them by \cref{assum:positive-bregman-ft}).
\end{proof}

\section{Proofs for \cref{sec:composite-objective}}
\label{apx:composite-proofs}

\begin{proof} \textbf{of \cref{cor:composite-obj-hidden-psi}.}
Define $\psi_0:=\psi_1$. Then, using our assumptions on $\psi_t$, we have
  \begin{align*}
    R^{(\ell)}_T(x^*)
    =
    &\
    R_T(x^*) + \sum_{t=1}^{T} (\psi_t(x_t) - \psi_t(x^*))
    \\
    =
    &\
    R_T(x^*) + \sum_{t=1}^{T}( \psi_t(x_{t+1}) - \psi_t(x^*) )+ \sum_{t=1}^{T} (\psi_{t}(x_{t}) - \psi_{t-1}(x_{t}) ) + \psi_1(x_1) - \psi_{T}(x_{T+1})
    \\
    \le
    &\ R_T(x^*) + \sum_{t=1}^{T} (\psi_t(x_{t+1}) - \psi_t(x^*))
    \\
    =
    &\
    R_T(x^*) + \sum_{t=1}^{T} \Big(q_t(x_{t+1}) - \tl{q}_t(x_{t+1}) + \tl{q}_t(x^*) - q_t(x^*) \Big)\,.
  \end{align*}
  The rest of the proof is as in \cref{cor:composite-obj-known-psi}, noting that $\tl{q}_0 = q_0$.
\end{proof}

\section{Proofs for \cref{sec:optimistic-learning}}
\label{apx:optimistic-learning-proofs}

\begin{proof} \textbf{of \cref{thm:ao-ftrl-bound}}.

  Starting from inequality \eqref{eq:put-together-ftrl}, by the exact same manipulations as in \cref{cor:recoveries}:
  \begin{align}
    R_T(x^*) \le &\ \sum_{t=0}^{T} \left( q_t(x^*) - q_{t}(x_{t+1}) \right) + \sum_{t=1}^{T} \dotx{g_t, x_{t} - x_{t+1}}
    \nonumber\\
    &\ + \sum_{t=1}^T \left( p_t(x^*) - p_t(x_t) \right) + \sum_{t=1}^{T} -\cB_{r_{1:t}}(x_{t+1}, x_t)
    \nonumber\\
    = &\ \sum_{t=0}^{T} \dotx{\tl{g}_{t+1} - \tl{g}_t, x^* - x_{t+1}} + \sum_{t=1}^{T} \dotx{g_t, x_{t} - x_{t+1}}
    \nonumber\\
    &\ + \sum_{t=0}^{T} \left( \tl{q}_t(x^*) - \tl{q}_{t}(x_{t+1}) \right) + \sum_{t=1}^T \left( p_t(x^*) - p_t(x_t) \right) + \sum_{t=1}^{T} -\cB_{r_{1:t}}(x_{t+1}, x_t)
    \nonumber\\
    = &\ \dotx{\tl{g}_{T+1}, x^*} + \sum_{t=0}^{T} \dotx{\tl{g}_{t+1}, -x_{t+1}} + \sum_{t=1}^{T} \dotx{\tl{g}_t, x_{t+1}} + \sum_{t=1}^{T} \dotx{g_t, x_{t} - x_{t+1}}
    \nonumber\\
    &\ + \sum_{t=0}^{T} \left( \tl{q}_t(x^*) - \tl{q}_{t}(x_{t+1}) \right) + \sum_{t=1}^T \left( p_t(x^*) - p_t(x_t) \right) + \sum_{t=1}^{T} -\cB_{r_{1:t}}(x_{t+1}, x_t)
    \nonumber\\
    = &\ \dotx{\tl{g}_{T+1}, x^* - x_{T+1}} + \sum_{t=1}^{T} \dotx{g_t - \tl{g}_t, x_{t} - x_{t+1}}
    \nonumber\\
    &\ + \sum_{t=0}^{T} \left( \tl{q}_t(x^*) - \tl{q}_t(x_{t+1}) \right) + \sum_{t=1}^T \left( p_t(x^*) - p_t(x_t) \right) + \sum_{t=1}^{T} -\cB_{r_{1:t}}(x_{t+1}, x_t)
    \label{eq:middle-of-ao-ftrl-proof}\\
    \le &\ \dotx{\tl{g}_{T+1}, x^* - x_{T+1}} + \tl{q}_T(x^*) - \tl{q}_T(x_{T+1})
    \nonumber\\
    &\ + \sum_{t=0}^{T-1} \left( \tl{q}_t(x^*) - \tl{q}_t(x_{t+1}) \right) + \sum_{t=1}^T \left( p_t(x^*) - p_t(x_t) \right) + \sum_{t=1}^{T} \frac{1}{2} \|g_t - \tl{g}_{t} \|_{(t),*}^2\,, \nonumber
  \end{align}
  using the Fenchel-Young for the second term, and \cref{assum:strong-cvx-regularizer} for the last term, in the final step. Finally, note that the left-hand side is independent of $\tl{q}_T$ and $\tl{g}_{T+1}$, and without loss of generality, we can set them to zero, which makes the first two terms of the right-hand side zero, hence finishing the proof.
\end{proof}

\begin{proof}\textbf{of \cref{thm:chiang-smooth-bound}}
  Define $G_t = \| g_t - \tl{g}_t \|_{*}^2$, and let $\lambda_t := \eta_t / 2$. Starting from \eqref{eq:middle-of-ao-ftrl-proof}, and using the fact that setting $\tl{g}_{T+1}=0$ does not affect the value of $R_T(x^*)$, we get
  \begin{align*}
    R_T(x^*) \le
    &\
    \sum_{t=1}^{T} \dotx{g_t - \tl{g}_t, x_{t} - x_{t+1}} + \sum_{t=1}^{T} -\cB_{r_{1:t}}(x_{t+1}, x_t)
    \nonumber\\
    &\ + \sum_{t=0}^{T} \left( \tl{q}_t(x^*) - \tl{q}_{t}(x_{t+1}) \right) + \sum_{t=1}^T \left( p_t(x^*) - p_t(x_t) \right)
    \\
    \le
    &\ \sum_{t=0}^{T} \left( \tl{q}_t(x^*) - \tl{q}_t(x_{t+1}) \right) + \sum_{t=1}^T \left( p_t(x^*) - p_t(x_t) \right)
    \\
    &\
    + \sum_{t=1}^{T} -\frac{\eta_t}{2} \|x_t - x_{t+1}\|^2 + \sum_{t=1}^{T} \frac{\lambda_t}{2} \|x_{t} - x_{t+1}\|^2 + \sum_{t=1}^T \frac{1}{2\lambda_t} \|g_t - \tl{g}_t\|_{*}^2\,,
    \\
    \le
    &\ \sum_{t=0}^{T} \left( \tl{q}_t(x^*) - \tl{q}_t(x_{t+1}) \right) + \sum_{t=1}^T \left( p_t(x^*) - p_t(x_t) \right)
    + \sum_{t=1}^{T} \frac{-\eta_t}{4} \| x_{t} - x_{t+1} \|^2 + \sum_{t=1}^{T} \frac{1}{\eta_t} G_t\,
    \\
    \le
    &\ \sum_{t=0}^{T} \left( \tl{q}_t(x^*) - \tl{q}_t(x_{t+1}) \right) + \sum_{t=1}^T \left( p_t(x^*) - p_t(x_t) \right)
    \\
    &\
    + \sum_{t=1}^{T} \frac{-2L^2}{\eta_{t+1}} \| x_{t} - x_{t+1} \|^2 + \sum_{t=1}^{T} \frac{1}{\eta_t} G_t\,.
  \end{align*}
  In the second inequality, we used \cref{assum:strong-cvx-regularizer} and the Fenchel-Young inequality.
  In the last inequality, we used the assumption $\eta_t \eta_{t+1} \ge 8 L^2$. Now, let $x_0 := x_1$ and $f_0 := 0$, so that $\tl{g}_1 = g_0 = \nabla f_0(x_0)$.
  Then, using ideas from Lemma 12 of \citet{chiang2012online},
  \begin{align*}
    \sum_{t=1}^{T} \frac{1}{\eta_t} G_t
    =
    &\
    \sum_{t=1}^{T} \frac{1}{\eta_t} \|\nabla f_t(x_t) - \nabla f_{t-1}(x_{t-1})\|_{*}^2
    \\
    \le
    &\
    \sum_{t=1}^{T} 2 \frac{1}{\eta_t} \|\nabla f_t(x_t) - \nabla f_{t-1}(x_{t})\|_{*}^2
    +
    \sum_{t=1}^{T} 2 \frac{1}{\eta_t} \|\nabla f_{t-1}(x_t) - \nabla f_{t-1}(x_{t-1})\|_{*}^2
    \\
    \le
    &\
    2 \sum_{t=1}^{T} \frac{1}{\eta_t} \|\nabla f_t(x_t) - \nabla f_{t-1}(x_{t})\|_{*}^2
    +
    \sum_{t=2}^{T}  \frac{2 L^2}{\eta_t} \| x_t - x_{t-1} \|^2
    \\
    \le
    &\
    2 \sum_{t=1}^{T} \frac{1}{\eta_t} \|\nabla f_t(x_t) - \nabla f_{t-1}(x_{t})\|_{*}^2
    +
    \sum_{t=1}^{T}  \frac{2 L^2}{\eta_{t+1}} \| x_{t+1} - x_{t} \|^2\,.
  \end{align*}
  Note that to get the first inequality, we used the fact that $\|\cdot\|^2$ is convex for any norm, together with Jensen's inequality, so that $\|x+y\|^2 = 4 \|x/2 + y/2\|^2 \le 4 (\|x\|^2 / 2 + \|y\|^2 / 2) = 2 \|x\|^2 + 2\|y\|^2$.
  This completes the proof.
\end{proof}

\begin{proof}\textbf{of \cref{cor:the-final-attack}.}

  First, note that since $\psi_t$ is convex, by definition, $r_{0:t} = \sum_{t=1}^{t} \frac{\eta_{t} - \eta_{t-1}}{2} \| \cdot - x_t\|^2$ is $\eta_t$-strongly-convex w.r.t. the norm $\| \cdot \|$, satisfying \cref{assum:strong-cvx-regularizer}. Furthermore, \cref{assum:positive-bregman-ft} is satisfied by the convexity of $f_t$.
  Also, by assumption, $\cX$ is closed and $R < +\infty$, so the objectives are always bounded below and \cref{assum:ftrl} holds.

  Let $G_t = \|g_t - \tl{g}_t\|_{*}^2$, and define $C := \eta_t - \eta \sqrt{G_{1:t}} = 4 R L^2$.
  Starting as in \cref{cor:composite-ao-ftrl}, and following the same steps as in the proof of \cref{thm:chiang-smooth-bound}, we have
  \begin{align*}
    R^{(\ell)}_T(x^*) \le
    &\
    \sum_{t=0}^{T} \left( \tl{q}_t(x^*) - \tl{q}_t(x_{t+1}) \right) + \sum_{t=1}^T \left( p_t(x^*) - p_t(x_t) \right)
    - \sum_{t=1}^{T} \frac{\eta_t}{4} \| x_{t} - x_{t+1} \|^2 + \sum_{t=1}^{T} \frac{1}{\eta_t} G_t\,
    \\
    \le
    &\
    \sum_{t=1}^T \frac{1}{2} (\eta_t - \eta_{t-1}) R^2
    + \sum_{t=1}^{T} \frac{-\eta_t}{4} \| x_{t} - x_{t+1} \|^2 + \sum_{t=1}^{T} \frac{1}{\eta \sqrt{G_{1:t}}} G_t\,
    \\
    \le
    &\
    \frac{1}{2} \eta_T R^2
    + \sum_{t=1}^{T} \frac{-C}{4} \| x_{t} - x_{t+1} \|^2 + \frac{2}{\eta} \sqrt{G_{1:T}}\,
    \\
    \le
    &\
    \frac{1}{2} C R^2 + \left( \frac{2}{\eta} + \frac{\eta}{2} R^2 \right) \sqrt{G_{1:T}} + \sum_{t=1}^{T} \frac{-C}{4} \| x_{t} - x_{t+1} \|^2 \,
    \\
  \end{align*}
  In the first line, we used, as in \cref{thm:chiang-smooth-bound}, the Fenchel-Young inequality with $\lambda_t = \eta_t/2$. In the second line, we dropped the $\tl{q}_t = 0$ terms and used the definition of $p_t$ and $R$, and the fact that $\eta_t \ge \eta_{t-1}$, to get the first term, and obtained the last term using the fact that $\eta_t \ge \eta \sqrt{G_{1:t}}$ by definition.
  In the third inequality, we let the $\eta_t$ in the first term telescope, used the fact that $\eta_t > C$ in the second term, and \cref{lem:adaptive-bound-lemma} to get the last term. In the last line, we used the definition of $\eta_T$ and grouped the $\sqrt{G_{1:T}}$ terms together.

  Next, we use the inequalities $\sqrt{a+b} \le \sqrt{a} + \sqrt{b}$ and $\sqrt{a} \le \frac{1}{2} + a$ (for $a,b \ge 0$), as well as Jensen's inequality on $\|\cdot\|^2$ (as in the proof of \cref{thm:chiang-smooth-bound}) to bound $\sqrt{G_{1:T}}$ with $\sqrt{D_{\| \cdot \|}}\,$:
  \begin{align*}
    \sqrt{\sum_{t=1}^{T} G_t }
    =
    &\
    \sqrt{\sum_{t=1}^{T} \|\nabla f_t(x_t) - \nabla f_{t-1}(x_{t-1})\|_{*}^2}
    \\
    \le
    &\
    \sqrt{\sum_{t=1}^{T} 2 \|\nabla f_t(x_t) - \nabla f_{t-1}(x_{t})\|_{*}^2
    +
    \sum_{t=1}^{T} 2 \|\nabla f_{t-1}(x_t) - \nabla f_{t-1}(x_{t-1})\|_{*}^2}
    \\
    \le
    &\
    \sqrt{2 \sum_{t=1}^{T} \|\nabla f_t(x_t) - \nabla f_{t-1}(x_{t})\|_{*}^2
    +
    \sum_{t=1}^{T} 2 L^2 \| x_t - x_{t-1} \|^2}
    \\
    \le
    &\
    \sqrt{2 \sum_{t=1}^{T} \|\nabla f_t(x_t) - \nabla f_{t-1}(x_{t})\|_{*}^2}
    +
    \sqrt{\sum_{t=1}^{T} 2 L^2 \| x_t - x_{t-1} \|^2}
    \\
    \le
    &\
    \sqrt{2 \sum_{t=1}^{T} \|\nabla f_t(x_t) - \nabla f_{t-1}(x_{t})\|_{*}^2}
    +
    \frac{1}{2} + \sum_{t=1}^{T} 2 L^2 \| x_t - x_{t-1} \|^2\,.
    \\
    =
    &\
    \sqrt{2 \sum_{t=1}^{T} \|\nabla f_t(x_t) - \nabla f_{t-1}(x_{t})\|_{*}^2}
    + \frac{1}{2} + \sum_{t=1}^{T} 2 L^2 \| x_{t+1} - x_{t} \|^2\,,
  \end{align*}
  where the last line follows, as in the proof of \cref{thm:chiang-smooth-bound}, by defining $x_0 = x_1$ and adding the extra positive term $2L^2 \|x_{T+1} - x_T\|^2$. Putting back into the previous inequality,
  \begin{align*}
    R^{(\ell)}_T(x^*) \le
    &\
    \frac{1}{2} C R^2 + \sum_{t=1}^{T} \frac{-C}{4} \| x_{t} - x_{t+1} \|^2
    \\
    &\
    + \left( \frac{2}{\eta} + \frac{\eta}{2} R^2 \right) \left(
    \sqrt{2 D_{\| \cdot \|}}
    + \frac{1}{2} + \sum_{t=1}^{T} 2 L^2 \| x_{t+1} - x_{t} \|^2
     \right) \,
    \\
    =
    &\
    \frac{1}{2} C R^2 + \sum_{t=1}^{T} \frac{-C}{4} \| x_{t} - x_{t+1} \|^2
    \\
    &\
    + 2 R \sqrt{2 D_{\| \cdot \|}}
    + R + \sum_{t=1}^{T} 4 R L^2 \| x_{t+1} - x_{t} \|^2 \,
    \\
    =
    &\
    \frac{1}{2} 4 R^3 L^2 + R + 2 R \sqrt{2 D_{\| \cdot \|}}\,.
  \end{align*}
  In the first equality, we used $\eta = 2/R$ while in the last one we used that $C=4RL^2$ by definition. This completes the proof.
\end{proof}

\begin{lemma}[Lemma 4 of \citet{mcmahan2014survey}]
  \label{lem:adaptive-bound-lemma}
  For any non-negative numbers $a_1, a_2, \dots, a_T$ with $a_1 > 0$,
  \begin{align*}
    \sum_{t=1}^{T} \frac{a_t}{\sqrt{\sum_{s=1}^{t} a_s}} \le 2 \sqrt{\sum_{t=1}^{T} a_t}\,.
  \end{align*}
\end{lemma}

\section{Technical results}\label{apx:tech-results}
In this appendix, we have gathered some technical results required in our proofs. The first lemma states that the Bregman divergence is invariant under addition of affine functions.
\begin{lemma}\label{lem:linear-invariance}
Let $f: \cH \to \ExReals$ be proper, and let $x,y \in \dom(f)$. Suppose that $v \in \cH$, and $w \in \Reals$, and let $g: \cH \to \ExReals$ be given by $g(\cdot) = f(\cdot) + \dotx{v,\cdot} + w$. Then,
\begin{enumerate}[(i)]
\item $g$ is proper, with $\dom(g) = \dom(f)$.
\item For any $z \in \cH$, the derivative $g'(x;z)$ exists in $[-\infty, +\infty]$ if and only if $f'(x;z)$ exists in $[-\infty, +\infty]$, in which case
\begin{align*}
	g'(x;z) = f'(x;z) + \dotx{v, z}\,.
\end{align*}
\item If $f'(x; y-x)$ or $g'(x;y-x)$ exist, then $\cB_{g}(y,x) = \cB_{f}(y,x)$.
\end{enumerate}
\end{lemma}
\begin{proof}
That $g$ is proper and $\dom(f) = \dom(g)$ is immediate since $\dom(\dotx{v,\cdot}) = \cH$ and $w \in \Reals$. Then $x,y \in \dom(g)$, and for any $z \in \cH$, if either of $f'(x;z)$ or $g'(x;z)$ exist in $[-\infty, +\infty]$,
\begin{align*}
f'(x;z) & + \dotx{v, z} = \lim_{\alpha \downarrow 0} \frac{f(x + \alpha z) + \dotx{v, x + \alpha z} + w - f(x) - \dotx{v, x} - w}{\alpha} = g'(x;v)\,,
\end{align*}
which proves the second part of the lemma. Letting $z = y-x$ and using the definition of $\cB_{g}$ gives $\cB_{f}(y,x) = \cB_{g}(y,x)$.
\end{proof}

The next proposition gathers useful results based on Proposition 17.2 of \citet{bauschke2011convex}.

\begin{proposition}\label{prop:convex-bregman-relation}
Let $f$ be proper and convex, and let $x,y \in \dom(f)$ and $z \in \cH$. Then,
\begin{enumerate}[(i)]
\item\label{itm:derivative-exists} $f'(x;z)$ exists in $[-\infty, +\infty]$  and
\begin{align*}
f'(x;z) = \inf_{\alpha \in (0, +\infty)} \frac{f(x+\alpha z) - f(x)}{\alpha}\,.
\end{align*}
\item\label{itm:derivative-not-infinite} $f'(x; y-x) < +\infty$\,.
\item\label{itm:bregman-nonnegative} $\cB_{f}(y,x) \ge 0$\,.
\end{enumerate}
\end{proposition}
\begin{proof}
Part (\ref{itm:derivative-exists}) is proved in Proposition 17.2(ii) of~\citet{bauschke2011convex}. Also, by their Proposition 17.2(iii),
\begin{align*}
	f'(x; y-x) + f(x) \le f(y)\,,
\end{align*}
proving part (\ref{itm:derivative-not-infinite}) since $f(y)$ and $f(x)$ are both real numbers. Part (\ref{itm:bregman-nonnegative}) then simply follows from the same equation, with the Bregman divergence being real and nonnegative when $f'(x; y-x)$ is real-valued, and $+\infty$ when $f'(x; y-x) = -\infty$.
\end{proof}

The next lemma is useful for decomposing Bregman divergences.
\begin{lemma}\label{lem:bregman-decomposition}
	Let $r:\cH \to \ExReals$ and $q:\cH \to \ExReals$ be proper and directionally differentiable. Let $S := \dom(r) \cap \dom(q)$, suppose $S \not= \emptyset$, and let $x,y \in S$. Suppose that at least one of the two limits $q'(x; y-x)$ and $r'(x; y-x)$ is finite. Then,
	\begin{align*}
		\cB_{r}(y,x) - \cB_{q}(y,x) = \cB_{p}(y,x)\,,
	\end{align*}
  where $p:\cH \to \ExReals$ is given by
  \begin{align*}
		p(x) :=
    \begin{cases}
      r(x) - q(x) & x \in \dom(r) \cup \dom(q)\,,
      \\
      +\infty & \text{otherwise}\,.
    \end{cases}
	\end{align*}
\end{lemma}
\begin{proof}
	By the assumption, we can add the two limits, to obtain
	\begin{align}
		- r'(x; y-x) + q'(x; y-x)
    &\ =
    \lim_{\alpha \downarrow 0} \frac{-r(x+ \alpha(y-x)) + r(x)}{\alpha} +
    \lim_{\alpha \downarrow 0} \frac{q(x+\alpha(y-x)) - q(x)}{\alpha}
    \nonumber \\
    &\ =
    \lim_{\alpha \downarrow 0} \frac{-p(x+ \alpha(y-x)) + p(x)}{\alpha}
    =
    - p'(x; y-x)\,.
    \label{eq:bregman1}
	\end{align}
  In the derivation above, we have used that at most one of $r(x+ \alpha(y-x))$ and $q(x+ \alpha(y-x))$ can remain infinite as $\alpha \downarrow 0$. Formally, there exists an $\epsilon > 0$ such that for all $\alpha < \epsilon$, the summation $- r(x+ \alpha(y-x)) + q(x+ \alpha(y-x))$ is well defined, and is equal, by definition, to $-p(x + \alpha(y-x))$.	Adding the real-valued equation $r(y) - r(x) - q(y) + q(x) = p(y) - p(x)$ to \eqref{eq:bregman1} completes the proof.
\end{proof}
In light of Proposition~\ref{prop:convex-bregman-relation} (\ref{itm:derivative-not-infinite}), if $p$ and $q$ are convex, then the limits are always less than $+\infty$, and the condition above is always satisfied.

\section{Proof of Theorem~\ref{thm:cheating-regret}}\label{apx:cheating-regret}
\newcommand{\hft}{h^{\tn{ftrl}}}
\newcommand{\hmd}{h^{\tn{md}}}
In this section, we provide a detailed proof of Theorem~\ref{thm:cheating-regret}.
First, we prove generalized versions of two lemmas that have appeared in several previous work; see, e.g., \citet{dekel2012optimal} and the references therein.

The first lemma is used for \AdaFTRL.
\begin{lemma}\label{lem:optimization-margin}
	Let $g \in \cH$ and consider a proper, directionally differentiable function $r:\cH \to \ExReals$. Define $S = \dom(r)$, and let $\cX \subset \cH$ be a convex set such that $\cX \cap S \not= \emptyset$. Further assume that $\argmin_{x \in \cX} \dotx{g, x} + r(x)$ is non-empty. Then, for any $x^+ \in \argmin_{x \in \cX} \dotx{g, x} + r(x)$ and any $x \in \cX \cap S$,
	\begin{align}
		+\infty > \dotx{g, x-x^+} + r(x) - r(x^+) \ge \cB_{r}(x, x^+)\,.
	\end{align}
\end{lemma}
\begin{proof}
	Let $h:\cH \to \ExReals$ be given by $h(\cdot) = \dotx{g, \cdot} + r(\cdot)$, so that $x^+ \in \argmin_{x \in \cX} h(x)$.	Note that by Lemma~\ref{lem:linear-invariance}, $\dom(h)=S$ and $h$ is directionally differentiable with $h'(x; z) = \dotx{g,z} + r'(x; z)$ for all $x \in S$ and $z \in \cH$. Also note that $x^+ \in \cX \cap S$ by definition. Since $x, x^+ \in \cX$, and $\cX$ is convex, for all $\alpha \in [0,1]$, we have $x^+ + \alpha (x-x^+) \in \cX$. Therefore, the optimality of $x^+$ over $\cX$ implies that for all $\alpha \in (0,1)$,
	\begin{align*}
		\frac{h(x^+ + \alpha (x-x^+)) - h(x^+)}{\alpha} \ge 0\,.
	\end{align*}
Thus, $0 \le h'(x^+; x-x^+) = \dotx{g,x-x^+} + r'(x^+; x-x^+)$, and therefore $+\infty > \dotx{g, x-x^+} \ge -r'(x^+; x-x^+)$. Adding the real number $r(x) - r(x^+)$ to the sides completes the proof.
\end{proof}

The second lemma is used for \AdaMD.

\begin{lemma}\label{lem:adamd-margin}
	Let $\cX, S, g$ and $r$ be as in Lemma~\ref{lem:optimization-margin}. Let $y \in S \cap \cX$ be such that $r'(y; \cdot - y)$ is real-valued and concave on $S$, i.e., for all $x_1, x_2 \in S$ and all $\alpha \in [0,1]$ for which $x_{\alpha} := x_1+\alpha(x_2-x_1) \in S$,
  \begin{align*}
    +\infty > r'(y; x_\alpha - y) \ge \alpha r'(y;x_2-y) + (1-\alpha) r'(y; x_1 - y) > - \infty\,.
  \end{align*}
 Let $q:\cH \to \ExReals$ be proper and directionally differentiable, with $S_q := S \cap \cX \cap \dom(q) \not= \emptyset$. Assume that $\cX^{+} := \argmin_{x \in \cX} \dotx{g, x} + q(x) + \cB_{r}(x,y)$ is non-empty, and the associated optimal value is finite. Then, for any $x^+ \in \cX^{+}$ and any $x \in S_q$,
	\begin{align}
		+\infty > \dotx{g, x-x^+} + q(x) - q(x^+) + \cB_{r}(x,y) - \cB_{r}(x^+,y) \ge \cB_{r+q} (x, x^+)\,.
	\end{align}
\end{lemma}
\begin{proof}
  Let $h:\cH \to \ExReals$ be given by $h(\cdot) = \dotx{g, \cdot} + q + \cB_{r}(\cdot, y)$, so that $x^+ \in \argmin_{x \in \cX} h(x)$. Note that by assumption, $h(x^+) < +\infty$. In addition, $\dom(h) \subset S \cap \dom(q)$. Thus, $x^+ \in S_q$. 

  Now, fix $\alpha \in (0,1)$, and let $x_{\alpha} = x^+ + \alpha(x-x^+)$. If $x_{\alpha} \in S_q$, then $q(x_{\alpha})$ and $r(x_{\alpha})$ are real-valued, and by the optimality of $x^+$ over $\cX$ and the concavity of $r'(y; \cdot - y)$ over $S$,
	\begin{align*}
		0 \le
    &\ h(x_{\alpha}) - h(x^+) = q(x_{\alpha}) - q(x^+) + \dotx{g, x^{+} + \alpha(x-x^{+}) - x^{+}} + \cB_{r}(x_{\alpha},y) - \cB_{r}(x^{+}, y)\\
  = &\ q(x_{\alpha}) - q(x^+) + \alpha \dotx{g, x-x^{+}} + r(x_{\alpha}) - r(x^{+})
  \\ &\  - r'(y; x_{\alpha} - y) + r'(y;x^{+} - y)\\
  \le &\ q(x_{\alpha}) - q(x^+) + \alpha \dotx{g, x-x^{+}} + r(x_{\alpha}) - r(x^{+})
  \\ &\ - \left( (1-\alpha) r'(y; x^{+} - y) + \alpha r'(y;x-y) \right) + r'(y;x^{+} - y)\\
  = &\ q(x_{\alpha}) - q(x^+) + \alpha \dotx{g, x-x^{+}} + r(x_{\alpha}) - r(x^{+}) + \alpha \left( r'(y; x^{+} - y) - r'(y;x-y) \right) \,,
	\end{align*}
  Suppose, on the other hand, that $x_{\alpha} \not\in S_q$. Then, given that by the assumption of convexity of $\cX$, $x_{\alpha} \in \cX$, we must have $x_{\alpha} \not\in S \cap \dom(q)$, so that $(r+q)(x_{\alpha}) = +\infty$. In addition, $r'(y; \cdot - y)$ is real-valued over $S$ and $x, x^+ \in S$, so $r'(y; x^{+} - y) - r'(y;x-y)$ is real-valued. Putting this together, we will again have that for $x_{\alpha} \not\in S_q$,
  \begin{align*}
		0 \le
    &\ q(x_{\alpha}) - q(x^+) + \alpha \dotx{g, x-x^{+}} + r(x_{\alpha}) - r(x^{+}) + \alpha \left( r'(y; x^{+} - y) - r'(y;x-y) \right) \,,
	\end{align*}

  Thus, dividing by the positive $\alpha$, for all $\alpha \in (0,1)$, we have
  \begin{align*}
    0 \le
    &\ \dotx{g, x-x^{+}} + \frac{q(x_{\alpha}) - q(x^+) + r(x_{\alpha}) - r(x^{+})}{\alpha} - r'(y;x-y) + r'(y; x^{+} - y)\,.
  \end{align*}
  Taking infimum over $\alpha$, we obtain
  \begin{align*}
    0 \le
    &\ \dotx{g, x-x^{+}} - r'(y;x-y) + r'(y; x^{+} - y) + \inf_{\alpha \in (0,1)} \frac{q(x_{\alpha}) - q(x^+) + r(x_{\alpha}) - r(x^{+})}{\alpha}
    \\
    \le
    &\ \dotx{g, x-x^{+}} - r'(y;x-y) + r'(y; x^{+} - y) + (r+q)'(x^+; x-x^+)\,,
  \end{align*}
  using directional differentiability of $q+r$ in the final step. Adding the real-valued equation $0 = q(x) - q(x^+) + r(x) - r(y) + r(y) - r(x^{+}) + (r+q)(x^{+}) - (r+q)(x)$, using the definition of Bregman divergence, and re-arranging terms completes the proof.
\end{proof}

We can now prove Theorem~\ref{thm:cheating-regret}.

\begin{proof}[Proof of Theorem~\ref{thm:cheating-regret}]
	First consider \AdaFTRL. For $t=0,1, \dots, T$, let $\hft_t: \cH \to \ExReals$ be given by $\hft_t(\cdot) := \dotx{g_{1:t}, \cdot} + q_{t}(\cdot) + r_{1:t}(\cdot)$, recalling that $c_{i:j} \equiv 0$ whenever $i>j$. Let $S_t = \dom(r_{1:t})$.
  By \cref{assum:ftrl}, for $t=1,2,\dots,T$,
  \begin{align*}
    -\infty < \hft_{t-1}(x_{t}) = \dotx{g_{1:t-1}, x_t} + q_{t-1}(x_t) + r_{1:t-1}(x_t)  < +\infty\,.
  \end{align*}
  Therefore, $x_t \in \dom(r_{1:t-1} + q_{t-1})$. In addition, by \eqref{eq:ptcond}, $x_t \in \dom(p_t)$. Thus, $x_t \in \dom(r_{1:t-1}+q_{t-1}+p_t)$, i.e., $x_t \in \dom(r_{1:t}) = S_t$. Furthermore, $\hft_t(x_{t+1}) < +\infty$, so $x_{t+1} \in \dom(q_t) \cap \dom(r_{1:t})$.
  Thus, $x_t, x_{t+1} \in \cX \cap S_t \subset \cap_{s=1}^{t} \left( \dom(q_{s-1}) \cap \dom(p_s) \right)$.

	Now, for any $t=1,2,\dots,T$, since $x_t$ minimizes $p_t$ over $\cX$, if we add $p_t$ to the objective of the optimization above, we will still have
	\begin{align*}
		x_{t} \in \argmin_{x \in \cX} \hft_{t-1}(x) + p_t(x) = \argmin_{x \in \cX} \dotx{g_{1:t-1},x} + r_{1:t}(x)\,.
	\end{align*}
	 By Assumption~\ref{assum:ftrl}, $r_{1:t}$ is directionally differentiable. Therefore, for any $t=1,2,\dots,T$, we can apply Lemma~\ref{lem:optimization-margin} with $g \gets g_{1:t-1}$, $r \gets r_{1:t}$, $x^+ \gets x_t$, and $x \gets x_{t+1}$, to obtain
	\begin{align*}
		\dotx{g_{1:t-1}, x_{t} - x_{t+1}} + p_{1:t}(x_{t}) - p_{1:t}(x_{t+1}) + q_{0:t-1}(x_{t}) - q_{0:t-1}(x_{t+1}) \le - \cB_{r_{1:t}}(x_{t+1}, x_t)\,.
	\end{align*}
	In the inequality above, the right-hand side cannot be equal to $-\infty$ (by Lemma~\ref{lem:optimization-margin}), and all other terms are real-valued. Thus, we can sum up this inequality over $t=1,2,\dots,T$, to obtain
	\begin{align*}
	\lefteqn{	- \sum_{t=1}^{T} \cB_{r_{1:t}}(x_{t+1}, x_t)} \\
	\ge &\ \sum_{t=1}^{T} \dotx{g_{1:t-1}, x_{t}} - \sum_{t=1}^{T} \dotx{g_{1:t-1}, x_{t+1}} + \sum_{t=1}^{T} p_{1:t}(x_{t}) - \sum_{t=1}^{T} p_{1:t}(x_{t+1}) + \\
  &\  \sum_{t=1}^{T} q_{0:t-1}(x_{t}) - \sum_{t=1}^{T} q_{0:t-1}(x_{t+1}) \\
  = &\ \sum_{t=0}^{T-1} \dotx{g_{1:t}, x_{t+1}} - \sum_{t=1}^{T} \dotx{g_{1:t-1}, x_{t+1}} + \sum_{t=1}^{T} p_{1:t}(x_{t}) - \sum_{t=2}^{T+1} p_{1:t-1}(x_{t}) + \\
  &\  \sum_{t=0}^{T-1} q_{0:t}(x_{t+1}) - \sum_{t=0}^{T} q_{0:t-1}(x_{t+1}) \\
= &\ \sum_{t=1}^{T-1} \dotx{g_{1:t}, x_{t+1}} - \sum_{t=1}^{T} \dotx{g_{1:t-1}, x_{t+1}} + \sum_{t=1}^{T} p_{1:t}(x_{t}) - \sum_{t=1}^{T+1} p_{1:t-1}(x_{t}) + \\
  &\  \sum_{t=0}^{T-1} q_{0:t}(x_{t+1}) - \sum_{t=0}^{T} q_{0:t-1}(x_{t+1}) \\
	= &\  - \dotx{g_{1:T}, x_{T+1}} + \sum_{t=1}^{T} \dotx{g_{t}, x_{t+1}} + \sum_{t=1}^{T} p_{t}(x_{t}) - p_{1:T}(x_{T+1}) - q_{0:T}(x_{T+1}) + \sum_{t=0}^{T} q_{t}(x_{t+1}) \\
	= &\  \sum_{t=1}^{T} \dotx{g_{t}, x_{t+1}} + \sum_{t=1}^{T} p_{t}(x_{t})
+ \sum_{t=0}^{T} q_{t}(x_{t+1}) - \Big( \dotx{g_{1:T}, x_{T+1}} + p_{1:T}(x_{T+1}) + q_{0:T}(x_{T+1})\Big) \\
	\ge &\ \sum_{t=1}^{T} \dotx{g_{t}, x_{t+1}} + \sum_{t=1}^{T} p_{t}(x_{t})
+ \sum_{t=0}^{T} q_{t}(x_{t+1}) - \Big( \dotx{g_{1:T}, x^*} + p_{1:T}(x^*) + q_{0:T}(x^*)\Big)\\
  = &\ R_T^{+}(x^*) + \sum_{t=1}^{T} p_{t}(x_{t})
+ \sum_{t=0}^{T} q_{t}(x_{t+1}) - p_{1:T}(x^*) - q_{0:T}(x^*)\,,
	\end{align*}
	using, in the last inequality, the optimality of $x_{T+1}$ over $\cX$, as well as the fact that $p_t,q_t$ are proper and all terms on the right-hand side not involving $x^*$ are real-valued (hence the term in the parentheses involving $x^*$ is well-defined and can be added to the rest of the expression). Now if $x^* \not\in \dom(p_{1:T} + q_{0:T})$, the bound of \cref{thm:cheating-regret} holds trivially (recalling that the Bregman divergences cannot be $+ \infty$). Otherwise, $(p_{1:T} + q_{0:T})(x^*)$ is real-valued, and rearranging completes the proof for \AdaFTRL.

  For \AdaMD, we start by presenting the implications of \cref{assum:md}.

  To simplify notation, let $x_0:=g_0:=0$, and define $\hmd_t := \dotx{g_t, x} + q_t(x) + \cB_{r_{1:t}}(x, x_{t})$ and $S_t = \dom(r_{1:t})$ for $t=0,1,\dots,T$ (so that $S_0 = \dom(r_{1:0}) = \cH$).
  Then, by \cref{assum:md}, $\hmd_t(x_{t+1}) < +\infty$ for all $t=0,1,\dots,T$, so $x_{t+1} \in \cX \cap \dom(q_t) \cap S_t$.
  Thus, given that $r_{1:t}'(x_t; \cdot - x_t)$ is real-valued on $S_t$, and $x_t \in S_t$ by assumption, $\cB_{r_{1:t}}(x_{t+1}, x_t)$ is also real-valued.

  Now, note that by the optimality of $x_{T+1}$, and because $\hmd_T(x_{T+1})$ is finite, for all $x^* \in \cX$,
  \begin{align}
    \dotx{g_{T}, x_{T+1} - x^*} \le q_T(x^*) - q_T(x_{T+1}) + \cB_{r_{1:T}}(x^*, x_T) - \cB_{r_{1:T}}(x_{T+1}, x_T)\,.
    \label{eq:md-proof-regret-at-T}
  \end{align}
  Next, fix $t \in \{0,1,2,\dots,T-1\}$ and suppose that $p_{t+1}(x^*)$ is finite-valued. Then, by the definition of $p_{t+1}$, we have $x^* \in \cX \cap \dom(q_t)$ and $x^* \in \dom(r_{t+1}) = \dom(r_{1:t+1}) \subset S_t$.
  Furthermore, by the argument above, $x_{t+1} \in \cX \cap S_t \cap \dom(q_t)$ and $x_t \in S_t$.
  Thus, for all $t=0,1,\dots,T-1$, we can apply \cref{lem:adamd-margin} with $g \gets g_t$, $r \gets r_{1:t}$, $q \gets q_t$, $y \gets x_t$, $x^{+} \gets x_{t+1}$, and $x \gets x^{*}$, to obtain
  \begin{align}
    \dotx{g_t, x_{t+1} - x^*} \le &\ q_t(x^*) - q_t(x_{t+1}) + \cB_{r_{1:t}}(x^*,x_t) - \cB_{r_{1:t}}(x_{t+1},x_t)
    \nonumber
    \\
    &\ - \cB_{r_{1:t}+q_t}(x^*,x_{t+1}) \,. \label{eq:md-proof-instant-regret}
  \end{align}
  Note that this also implies that the right-hand side above cannot be $-\infty$, and only the last Bregman divergence term could be infinite. Now, since $r'_{1:t+1}(x_{t+1}; \cdot - x_{t+1})$ is real-valued on $S_{t+1}$, $r_{1:t} + q_t$ is directionally differentiable, and $x^* \in S_{t+1}$, by \cref{lem:bregman-decomposition} we have
  \begin{align*}
    \cB_{r_{1:t+1}}(x^*, x_{t+1}) - \cB_{r_{1:t} + q_t}(x^*, x_{t+1}) = \cB_{p_{t+1}}(x^*, x_{t+1})\,.
  \end{align*}
  In particular, this implies that $\cB_{p_{t+1}}(x^*, x_{t+1})$ cannot be $-\infty$.
  Moving the (real-valued) first term to the right-hand side, and substituting into \eqref{eq:md-proof-instant-regret}, we have
  \begin{align*}
    \dotx{g_t, x_{t+1} - x^*} \le &\ q_t(x^*) - q_t(x_{t+1}) + \cB_{r_{1:t}}(x^*,x_t) - \cB_{r_{1:t}}(x_{t+1},x_t) + \cB_{p_{t+1}}(x^*, x_{t+1}) \\
    &\ - \cB_{r_{1:t+1}}(x^*, x_{t+1}) \,.
  \end{align*}

  In light of the above, if $p_{t+1}(x^*)$ is finite-valued for all $t=0,1,2,\dots,T-1$, then summing up the above inequality, as well as \eqref{eq:md-proof-regret-at-T}, we have
  \begin{align*}
    \sum_{t=0}^{T} \dotx{g_t, x_{t+1} - x^*}
    \le &\ \sum_{t=0}^{T-1} q_t(x^*) - q_t(x_{t+1}) + \sum_{t=0}^{T-1} \cB_{p_{t+1}}(x^*,x_{t+1}) +
    \\
    &\ \sum_{t=0}^{T-1} \cB_{r_{1:t}}(x^*,x_t) - \sum_{t=0}^{T-1} \cB_{r_{1:t+1}}(x^*,x_{t+1}) - \sum_{t=0}^{T-1} \cB_{r_{1:t}}(x_{t+1},x_t) + \\
    &\ q_T(x^*) - q_T(x_{T+1}) + \cB_{r_{1:T}}(x^*,x_T) - \cB_{r_{1:T}}(x_{T+1}, x_T)
    \\
    = &\ \sum_{t=0}^{T} q_t(x^*) - q_t(x_{t+1}) + \sum_{t=0}^{T-1} \cB_{p_{t+1}}(x^*,x_{t+1}) +
    \\
    &\ \sum_{t=}^{T} \cB_{r_{1:t}}(x^*,x_t) - \sum_{t=0}^{T-1} \cB_{r_{1:t+1}}(x^*,x_{t+1}) - \sum_{t=1}^{T} \cB_{r_{1:t}}(x_{t+1},x_t)
    \\
    = &\ \sum_{t=0}^{T} q_t(x^*) - q_t(x_{t+1}) + \sum_{t=1}^{T} \cB_{p_{t}}(x^*,x_{t}) - \sum_{t=1}^{T} \cB_{r_{1:t}}(x_{t+1},x_t)\,,
    \\
  \end{align*}
  and \eqref{eq:md-forward-regret} holds.

  On the other hand, if $p_{t+1}(x^*)$ is infinite for at least one $t$ in $\{0,1,2,\dots,T-1\}$, then $\cB_{p_{t+1}}(x^*, x_{t+1}) = +\infty$ by definition.
  Therefore, the right-hand side of \eqref{eq:md-forward-regret} will be $+\infty$, given that by the argument above, $\cB_{p_{t+1}}(x^*, x_{t+1})$ cannot be equal to $-\infty$ if $p_{t+1}(x^*)$ is finite-valued.
  Thus, in this case as well, the bound of \eqref{eq:md-forward-regret} holds trivially, completing the proof.
\end{proof}

\end{document}